%% file: paper.tex
\documentclass[jair,twoside,11pt]{article}
\usepackage{jair, rawfonts}

\usepackage{theapa}

\usepackage[utf8]{inputenc} % allow utf-8 input
\usepackage[T1]{fontenc}    % use 8-bit T1 fonts
\usepackage{url}            % simple URL typesetting
\usepackage{booktabs}       % professional-quality tables
\usepackage{amsfonts}       % blackboard math symbols
\usepackage{nicefrac}       % compact symbols for 1/2, etc.
\usepackage{microtype}      % microtypography

\usepackage[colorinlistoftodos]{todonotes}

\usepackage{graphicx} % more modern
\usepackage{subcaption}
\usepackage{sidecap}
\usepackage{tabularx}
\usepackage{pbox}
\usepackage{csquotes}
\usepackage{booktabs}
\usepackage{xcolor}
\usepackage{wrapfig}

\newcommand*\citep{\cite}

\usepackage{mathtools}

\usepackage{caption}
\usepackage{algorithm}
\usepackage{algorithm,algorithmicx}
\usepackage[noend]{algpseudocode}

\usepackage{enumitem}

\usepackage{amsmath, amssymb}

\usepackage{amsthm}
\usepackage{thmtools,thm-restate}

\newtheorem{theorem}{Theorem}

\newtheorem{proposition}{Proposition}
\newtheorem{corollary}{Corollary}

\newtheorem{assumption}{Assumption}

\newcommand{\Naturals}{\mathbb{N}}
\newcommand{\Reals}{\mathbb{R}}
\newcommand{\RR}{\mathbb{R}}

\newcommand{\Expected}{\mathbb{E}}
\newcommand{\defeq}{\mathrel{\overset{\makebox[0pt]{\mbox{\normalfont\tiny\sffamily def}}}{=}}}
\newcommand{\trans}{\top}

\newcommand{\inv}{{\raisebox{.2ex}{$\scriptscriptstyle-1$}}}

\newcommand{\partialderivative}[2]{\frac{\partial #1}{\partial #2}}
\newcommand{\pd}[2]{\frac{\partial #1}{\partial #2}}
\newcommand{\twovec}[2]{{\small \Big[\!\! \begin{array}{c}
#1\\
#2
\end{array}
\!\!\Big] }}

\newcommand{\Ppi}{\mathbf{P}^\pi}

\newcommand{\Cpij}[1]{\mathbf{C}^{(#1)}}

\newcommand{\Ppigammaj}[1]{\mathbf{P}^{(#1)}}
\newcommand{\pij}[1]{\pi^{(j)}}
\newcommand{\gammaj}[1]{\gamma^{(#1)}}
\newcommand{\cj}[1]{c^{(#1)}}

\newcommand{\vit}[2]{\mathbf{V}^{(#1)}_{#2}}
\newcommand{\vi}[1]{\mathbf{V}^{(#1)}}
\newcommand{\vt}[1]{\mathbf{V}_{#1}}
\newcommand{\vinone}{\mathbf{V}}
\newcommand{\Bn}{\mathbf{B}}

\newcommand{\vifunc}[1]{V^{(#1)}}
\newcommand{\viweights}[1]{V^{(#1)}_\weights}

\newcommand{\evec}{\mathbf{e}}
\newcommand{\gvec}{\mathbf{g}}
\newcommand{\hvec}{\mathbf{h}}

\newcommand{\svec}{\mathbf{s}}

\newcommand{\vvec}{\mathbf{v}}
\newcommand{\wvec}{\mathbf{w}}
\newcommand{\xvec}{\mathbf{x}}

\newcommand{\zvec}{\mathbf{z}}

\newcommand{\zerovec}{\mathbf{0}}

\newcommand{\phivec}{\boldsymbol{\phi}}
\newcommand{\psivec}{\boldsymbol{\psi}}

\newcommand{\phimat}{\Phi}

\newcommand{\Actions}{\mathcal{A}}
\newcommand{\action}{a}
\newcommand{\States}{\mathcal{S}}

\newcommand{\cumulant}{c}
\newcommand{\state}{z}

\newcommand{\Pfcn}{\mathrm{P}}

\newcommand{\nsamples}{m}

\newcommand{\dw}{\mathbf{d}}
\newcommand{\dwdiag}{\mathbf{D}}

\newcommand{\tpolicy}{\pi}

\newcommand{\obs}{\mathbf{o}}
\newcommand{\Observations}{\mathcal{O}}

\newcommand{\Hists}{\mathcal{H}}

\newcommand{\Hist}{\mathcal{H}}

\newcommand{\weights}{{\boldsymbol{\theta}}}
\newcommand{\secweights}{\wvec}
\newcommand{\weightspace}{\Theta}

\newcommand{\xsize}{d}
\newcommand{\obssize}{m}
\newcommand{\valuesize}{n}
\newcommand{\numgvfs}{n}
\newcommand{\statesize}{\numgvfs}
\newcommand{\featuresize}{(\xsize+\valuesize)}

\newcommand{\Roperator}{\mathcal{R}}
\newcommand{\tderror}{\delta}

\newcommand{\actdot}{\dot{\sigma}}
\newcommand{\actdotdot}{\ddot{\sigma}}

\newcommand{\cfunc}{c}

\newcommand{\yhat}{\hat{y}}

\DeclareMathOperator*{\diag}{diag}
\newcommand{\valuedr}{\boldsymbol{\xi}}
\newcommand{\valuedtheta}{\mathbf{u}}
\newcommand{\RopValueVect}{({\mathcal{R}_t})_{w,V}}
\newcommand{\krondelta}{\delta^\kappa}

\newcommand{\expect}{\operatorname{\mathbb{E}}\expectarg}
\DeclarePairedDelimiterX{\expectarg}[1]{[}{]}{%
\ifnum\currentgrouptype=16 \else\begingroup\fi
\activatebar#1
\ifnum\currentgrouptype=16 \else\endgroup\fi
}

\newcommand{\innermid}{\nonscript\;\delimsize\vert\nonscript\;}
\newcommand{\activatebar}{%
\begingroup\lccode`\~=`\|
\lowercase{\endgroup\let~}\innermid
\mathcode`|=\string"8000
}

\jairheading{70}{2021}{497-543}{04/2020}{01/2021}
\ShortHeadings{General Value Function Networks}
{Schlegel, Jacobsen, Abbas, Patterson, White \& White}
\firstpageno{497}

\title{General Value Function Networks}
\author{\name Matthew Schlegel \email mkschleg@ualberta.ca \\
	\name Andrew Jacobsen \email ajjacobs@ualberta.ca  \\
	\name Zaheer Abbas \email mzaheer@ualberta.ca \\
	\name Andrew Patterson \email ap3@ualberta.ca  \\
        \name Adam White \email amw@ualberta.ca \\
      	\name Martha White \email whitem@ualberta.ca \\
        \addr Department of Computing Science and the Alberta Machine Intelligence Institute (Amii)\\
        University of Alberta, Canada\\
}

\begin{document}

\maketitle

\begin{abstract}
State construction is important for learning in partially observable environments. A general purpose strategy for state construction is to learn the state update using a Recurrent Neural Network (RNN), which updates the internal state using the current internal state and the most recent observation. This internal state provides a summary of the observed sequence, to facilitate accurate predictions and decision-making. At the same time, specifying and training RNNs is notoriously tricky, particularly as the common strategy to approximate gradients back in time, called truncated Back-prop Through Time (BPTT), can be sensitive to the truncation window. Further, domain-expertise---which can usually help constrain the function class and so improve trainability---can be difficult to incorporate into complex recurrent units used within RNNs. In this work, we explore how to use multi-step predictions to constrain the RNN and incorporate prior knowledge. In particular, we revisit the idea of using predictions to construct state and ask: does constraining (parts of) the state to consist of predictions about the future improve RNN trainability? We formulate a novel RNN architecture, called a General Value Function Network (GVFN), where each internal state component corresponds to a prediction about the future represented as a value function. We first provide an objective for optimizing GVFNs, and derive several algorithms to optimize this objective. We then show that GVFNs are more robust to the truncation level, in many cases only requiring one-step gradient updates.
\end{abstract}

\section{Introduction}

Most domains of interest are partially observable, where an agent only observes a limited part of the state. In such a setting, if the agent uses only the immediate observations, then it has insufficient information to make accurate predictions or decisions. A natural approach to overcome partial observability is for the agent to maintain a history of its interaction with the world. For example, consider an agent in a large and empty room with low-powered sensors that reach only a few meters. In the middle of the room, with just the immediate sensor readings, the agent cannot know how far it is from a wall. Once the agent reaches a wall, though, it can determine its distance from the wall in the future by remembering this interaction. This simple strategy, however, can be problematic if a long history length is needed \citep{mccalum1996learning}.

State construction enables the agent to overcome partial observability, with a more compact representation than an explicit history. Because most environments and datasets are partially observable---in time series prediction, in modeling dynamical systems and in reinforcement learning---there is a large literature on state construction. These strategies can be separated into Objective-state and Subjective-state approaches.

Objective-state approaches specify a true latent space, and use observations to identify this latent state. An objective representation is one that is defined in human-terms, external to the agent's data-stream of interaction. They typically require an expert to provide feature generators or models of the agent's motion and sensor apparatus. Many approaches are designed for a discrete set of latent states, including HMMs \citep{baum1966statistical} and POMDPs \citep{kaelbling1998planning}.
A classical example is Simultaneous Localization and Mapping, where the agent attempts to extract its position and orientation as a part of the state \citep{durrantwhyte2006simultaneous}.
These methods are particularly useful in applications where the dynamics are well-understood or provided, and so accurate transitions can be used in the explicit models. When models need to be estimated or the latent space is unknown, however, these methods either cannot be applied or are prone to misspecification.

The goal of subjective-state approaches, on the other hand, is to construct an internal state only from a stream of experience. This contrasts objective-state approaches in two key ways. First, the agent is not provided with a true latent space to identify. Second, the agent need not identify a true latent state, even if there is one. Rather, it only needs to identify an internal state that is sufficient for making predictions about target variables of interest. Such a state will likely not correspond to objective quantities like meters and angles, but could be much simpler than the true latent state and can be readily learned from the data stream. Examples of subjective-state approaches to state construction include Recurrent Neural Networks (RNNs) \citep{hopfield1982neural,lin1993reinforcement}, Predictive State Representations (PSRs) \citep{littman2001predictive} and TD Networks \citep{sutton2004temporal}.

RNNs have emerged as one of the most popular approaches for online state construction, due to their generality and the ability to leverage advances in optimizing neural networks. An RNN provides a recurrent state-update function, where the state is updated as a function of the (learned) state on the previous step and the current observations. These recurrent connections can be unrolled back in time, making it possible for the current RNN state to be dependent on observations far back in time. There have been several specialized activation units crafted to improve learning long-term dependencies, including long short-term memory units (LSTMs) \citep{hochreiter1997long} and gated recurrent units (GRUs) \citep{cho2014properties}. PSRs and TD Networks are not as widely used, because they make use of complex training algorithms that do not work well in practice (see \citeR{mccracken2005online,boots2011closing} and \citeR{vigorito2009temporal,silver2012gradient} respectively). In fact, recent work has investigated facilitating use of these models by combining them with RNNs \citep{downey2017predictive,choromanski2018initialization,venkatraman2017predictive}. Other subjective state approaches based on filtering can be complicated to extend to nonlinear dynamics, such as system identification approaches \citep{ljung2010perspectives} or Predictive Linear Gaussian models \citep{rudary2005predictive,wingate2006mixtures}.

One issue with RNNs, however, is that training can be unstable and expensive. There are two well-known approaches to training RNNs. The first, Real Time Recurrent Learning (RTRL) \citep{williams1989alearning} relies on a recursive form to estimate gradients. This gradient computation is exact in the offline setting---when RNN parameters are fixed---but only an approximation when computing gradients online. RTRL is prohibitively expensive, requiring computation that is quartic in the hidden dimension size $\statesize$. Low-rank approximations have been developed \citep{tallec2018unbiased,mujika2018approximating,benzing2019optimal} to improve computational efficiency, but these approaches to training RNNs remain less popular than the simpler strategy of Back propagation through time (BPTT).

BPTT explicitly computes gradients of the parameters, by using the chain rule back in time, essentially unrolling the recursive RNN computation. This approach requires maintaining the entire trajectory, which is infeasible for many online learning systems we consider here. A truncated form of BPTT (p-BPTT) is often used to reduce the complexity of training, where complexity grows linearly with p: $O(p \statesize^2)$.
Unfortunately, training can be highly sensitive to the truncation parameters \citep{pascanu2013onthe}, particularly if the dependencies back-in-time are longer than the chosen $p$---as we reaffirm in our experiments.

One potential cause of this instability is precisely the generality of RNNs. These systems require expertise in selecting architectures and tuning hyperparameters \citep{pascanu2013onthe,sutskever2013training}. This design space can already be difficult to navigate with standard feed-forward neural networks, and is exacerbated by the recurrence that makes the learning dynamics more unstable. Further, it can be hard to leverage domain expertise to constrain the space of RNNs, and so improve trainability. Specialized, complex architectures have been designed for speech recognition \citep{saon2017english} and NLP \citep{Peters:2018}; redesigning such systems for new problems is an onerous task. Many general purpose architectural restrictions have been proposed, such as GRUs and skip connections (see \citeR{greff2017lstm} and \citeR{trinh2018learning} for thorough overviews). These methods all provide tools to design, and tune, better architectures, but still do not provide a simple mechanism for a non-expert in deep learning to inject prior knowledge.

An alternative direction, that requires more domain expertise than RNN expertise, is to use predictions as auxiliary losses. Auxiliary unsupervised losses have been used in NLP to improve trainability \citep{trinh2018learning}. Less directly, auxiliary losses were used in reinforcement learning \citep{jaderberg2016reinforcement} and for modeling dynamical systems \citep{venkatraman2017predictive}, to improve the quality of the representation; this is a slightly different but nonetheless related goal to trainability. The use of predictions for auxiliary losses is an elegant way to constrain the RNN, because the system designers are likely to have some understanding of the relevant system components to predict. For the larger goals of AI, augmenting the RNN with additional predictions is promising because one could imagine the agent discovering these predictions autonomously---predictions by design are grounded in the data stream and learnable without human supervision. Nonetheless, the use of predictions as auxiliary tasks provides a more indirect (second-order) mechanism to influence the state variables. In this work, we ask: is there utility in directly constraining states to be predictions?

To answer this question, we need a practical approach for learning RNNs, where the internal state corresponds to predictions. We propose a new RNN architecture, where we constrain the hidden state to be multi-step predictions, using an explicit loss function on the hidden state.
In particular, we use general policy-contingent, multi-step predictions---called General Value Functions (GVFs) \citep{sutton2011horde}---generalizing the types of predictions considered in related predictive representation architectures \citep{rafols2005using,silver2012gradient,sun2016learning,downey2017predictive}. These GVFs have been shown to represent a wide array of multi-step predictions
\citep{modayil2014multi}. In this paper, we develop the objective and algorithm(s) to train these GVF networks (GVFNs).

We then demonstrate through a series of experiments that GVFNs can effectively represent the state and are much more robust to train, allowing even simple gradient updates with no gradients needed
back-in-time. We first investigate accuracy on two time series datasets, and find that our approach is
competitive with a baseline RNN and more robust to BPTT truncation length. We then investigate GVFNs more deeply in several synthetic problems, to determine 1) if robustness
to truncation remains for a domain with long-term dependencies and 2) the impact of the prediction specification---or misspecification---on GVFN performance. We find that GVFNs have consistent robustness properties across problems, but that, unsurprisingly, the choice of predictions do matter, both for improving learning as well as final accuracy. Our experiments provide evidence that constraining states to be predictions can be effective, and raise the importance of better understanding what these predictions should be.

Our work provides additional support for the {\em predictive representation hypothesis}, that state-components restricted to be predictions about the future result in good generalization \citep{rafols2005using}. Constraining the state to be predictions could both regularize learning---by reducing the hypothesis space for state construction---and prevent the constructed state from overfitting to the observed data and target predictions. To date, there has only been limited investigation into and evidence for this hypothesis.
\citeA{rafols2005using} showed that, for a discrete state setting, learning was more sample efficient with a predictive representation than a tabular state representation and a tabular history representation.
\citeA{schaul2013better} showed how a collection of optimal GVFs---learned offline---provide a better state representation for a reward maximizing task, than a collection of optimal PSR predictions.
\citeA{sun2016learning} showed that, for dynamical systems, constraining state to be predictions about the future significantly improved convergence rates over auto-regressive models and n4sid.
Our experiments show that RNNs with state composed of GVF predictions can have notable advantage over RNNs in building state with p-BPTT, even when the RNN is augmented with auxiliary tasks based on those same GVFs.

\section{Problem Formulation}

We consider a partially observable setting, where the observations are a function of an unknown, unobserved underlying state.
The dynamics are specified by transition probabilities $\Pfcn = \States \times \Actions \times \States \rightarrow [0,\infty)$ with state space $\States$ and action-space $\Actions$. On each time step the agent receives an observation vector $\obs_t \in \Observations \subset \Reals^\obssize$, as a function $\obs_t = \obs(\state_t)$ of the underlying state $\state_t \in \States$. The agent only observes $\obs_t$, not $\state_t$, and then takes an action $\action_t$, producing a sequence of observations and actions: $\obs_{0}, a_{0}, \obs_{1}, a_1, \ldots$.

The goal for the agent under partial observability is to identify a state representation $\svec_t \in \RR^\numgvfs$ which is a sufficient statistic (summary) of past interaction, for targets $y_t$. More precisely, such a \emph{sufficient state} ensures that $y_t$ given this state is independent of history $\hvec_t = \obs_0, a_{0}, \obs_1, a_1, \ldots, \obs_{t-1}, a_{t-1}, \obs_{t}$,
\begin{equation}
  p(y_{t} | \svec_t) = p(y_{t} | \svec_t, \hvec_t)
\end{equation}
or so that statistics about the target are independent of history, such as $\mathbb{E}[Y_{t} | \svec_t] = \mathbb{E}[Y_{t} | \svec_t, \hvec_t]$.
Such a state summarizes the history, removing the need to store the entire (potentially infinite) history.
Note here that this is a less stringent definition of sufficient state than used for PSRs \citep{littman2001predictive}, where the state is constructed for predictions about all future outcomes. We presume that the agent has a limited set of targets of interest, and needs to find a sufficient state for just those targets. For example, a potential set of targets is the observation vector on the next time step.

One strategy for learning such a state is with \emph{recurrent neural networks} (RNNs), which learn a state-update function. Imagine a setting where the agent has a sufficient state $\svec_t$ for this step. To obtain sufficient state for the next step, it simply needs to update $\svec_t$ with the new information in the given observation and action $\xvec_{t+1} = [a_t, \obs_{t+1}] \in \RR^{\xsize}$. The goal, therefore, is to learn a state-update function $f: \RR^{\statesize+\xsize} \rightarrow \RR^{\statesize}$ such that
\begin{equation}
\svec_{t+1} = f(\svec_t, \xvec_{t+1}) \label{eq_update}
\end{equation}
provides a sufficient state $\svec_{t+1}$.
The update function $f$ is parameterized by a weight vector $\weights \in \weightspace$ in some parameter space $\weightspace$.
An example of a simple RNN update function, for $\weights$ composed of stacked vectors $\weights^{(j)} \in \RR^{\statesize+\xsize}$ for each hidden state $j \in \{1, \ldots, \statesize\}$ is, for activation function $\sigma: \RR \rightarrow \RR$,
\begin{figure*}[h!]
\centering
\begin{minipage}{0.3\textwidth}
\small
\begin{equation*}
\svec_{t+1} = \left[\begin{array}{c}
\sigma\Big(\twovec{\svec_{t}}{\xvec_{t+1}}^\top \weights^{(1)} \Big)\\
\vdots\\
\sigma\Big(\twovec{\svec_{t}}{\xvec_{t+1}}^\top \weights^{(\statesize)} \Big)
\end{array} \right]
\end{equation*}
\normalsize
\end{minipage}
\hspace{-1.5cm}
\begin{minipage}{0.3\textwidth}
\hspace{2cm} depicted as
\end{minipage}
\begin{minipage}{0.4\textwidth}
    \includegraphics[width=\textwidth]{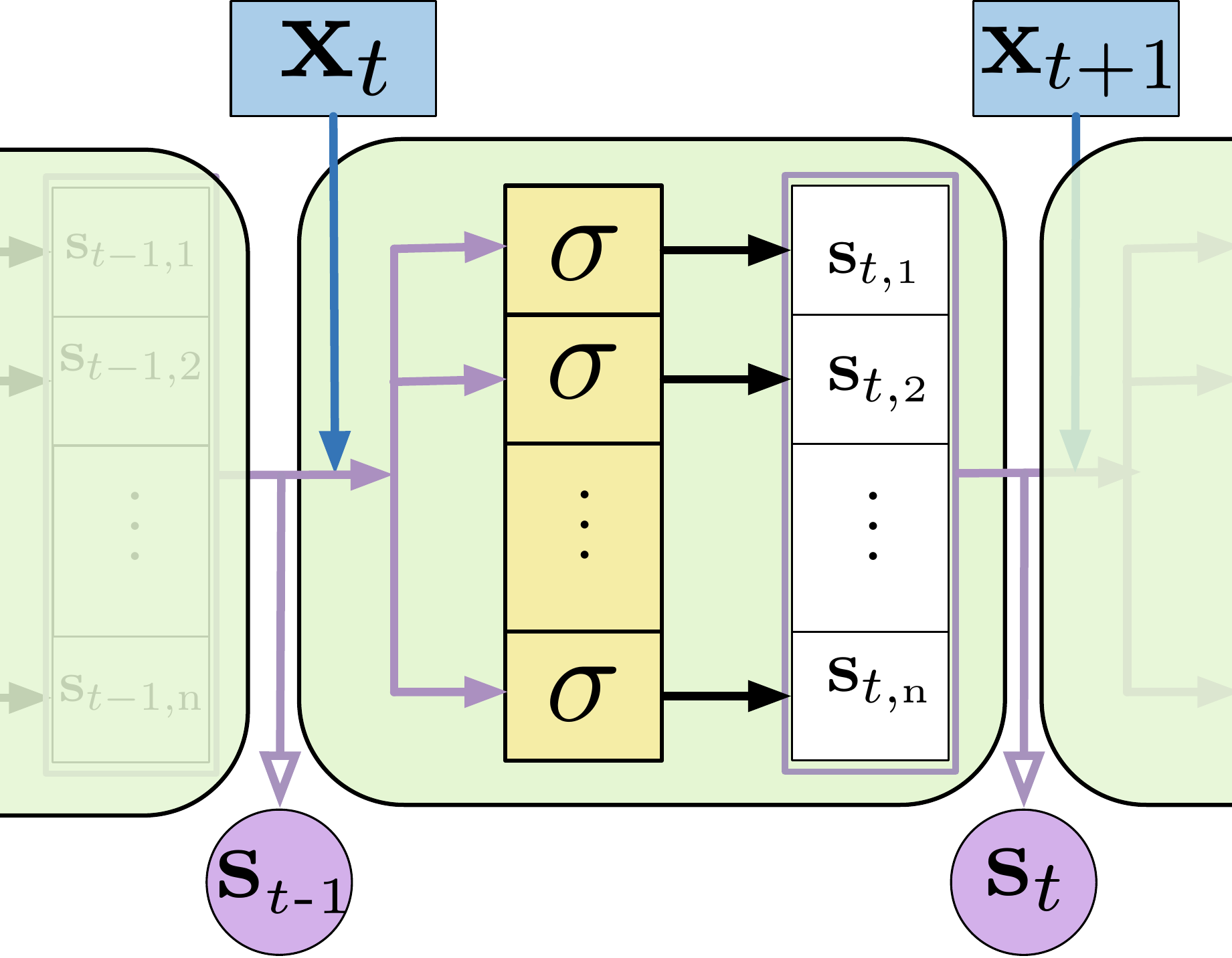}
\end{minipage}
\end{figure*}

\noindent In many cases, learning a sufficient state under function approximation may not be possible. Instead, this state is approximated so as to improve prediction accuracy of the target $y_{t}$.

The goal in this work is to develop an efficient algorithm to learn this state-update function. Most RNN algorithms learn this state-update by minimizing prediction error to desired targets $y_t \in \RR$ across time. For example, for $\yhat_t = \svec_t^\top \wvec$ for  weights $\wvec \in \RR^{\statesize}$, the loss for $\weights$ on time step $t$ could be
\begin{align*}
\ell(\yhat_t, y_t) &\defeq (\yhat_t - y_t)^2 =  (\svec_t^\top \wvec - y_t)^2 \\
&= (f_\weights(\svec_{t-1}, \xvec_t)^\top \wvec - y_t)^2
= \Big(f_\weights\big( f_\weights(\svec_{t-2}, \xvec_{t-1}), \xvec_{t}\big)^\top \wvec - y_t \Big)^2 = \ldots
\end{align*}
This objective, however, can be difficult to optimize. The weights $\weights$ can influence the state variables far back in time, with small changes for early states resulting in big changes to the state many steps later. This sensitivity to the weights can result in both vanishing and exploding gradient problems \citep{pascanu2013onthe}. Even worse, the problem is under-constrained, particularly if there is a scalar target. The loss may encourage the weights to change the immediate state $\svec_t$ quite a bit---just to reduce error for this single stochastic target---resulting in potentially destabilizing changes to the weights that influence states all the way back in time.

One strategy is to consider how to constrain the state variables, so as to avoid changes just due to the targets, and stochasticity in the targets. We pursue a strategy, inspired by predictive representations, where the state-update function is learned such that each hidden state is an accurate prediction about future outcomes, described in the next section.

\section{Constraining State to be Predictions}\label{sec_constraining}

Let us start in a simpler setting and explain how the hidden units could be trained to be n-horizon predictions about the future. Imagine you have a multi-dimensional time series of a power-plant, consisting of $d$ sensory observations with the first sensor corresponding to water temperature. Your goal is to make a hidden node in your RNN predict the water temperature in 10 steps, because you think this feature is useful to make other predictions about the future.

This can be done simply by adding the following loss: $(\svec_{t,1} - \xvec_{t+10, 1})^2$. The combined loss $L_t(\weights)$ on time step $t$ is
\begin{equation}
L_t(\weights) \defeq
\ell(\yhat_t, y_t) +  (\svec_{t,1} - \xvec_{t+10, 1})^2
\end{equation}
where both $\yhat_t$ and $\svec_t$ are implicitly functions of $\weights$.
This loss still encourages the RNN to find a hidden state $\svec_t$ that predicts $y_t$ well. There is likely a whole space of solutions that have similar accuracy for this prediction. The second loss constrains this search to pick a solution where the first state node is a prediction about an observation 10 steps into the future. This second term can be seen as a regularizer on the network, specifying a preference on the learned solution.
In general, more than one state node---even all of $\svec_t$---could be learned to be predictions about the future.

The difficulty in training such a state depends on the chosen targets. For example, long horizon targets---such as 100 steps rather than 10 steps into the future---can be high variance. Even if such a predictive feature could be useful, it may be difficult to learn accurately and could make the state-update less stable. Using n-horizon predictions also requires a delay in the update: the agent must wait 100 steps to see the target to update the state at time $t$.

We therefore propose to restrict ourselves to a class of prediction that have
been shown to be more robust to these issues
\citep{van2015learning,sutton2011horde,modayil2014multi}. This class of
predictions correspond to predictions of discounted cumulative sums of signals
into the future, called General Value Functions (GVFs). We have algorithms to
estimate these predictions online, without having to wait to see outcomes in the
future. This property of GVFs is called \emph{independence of span} \citep{van2015learning}, meaning learning can be achieved with computation and memory independent of the horizon. Such a property is doubly critical for predictions within an RNN, as it is more likely that we can actually learn these predictions sufficiently quickly to be usable as state.
% MARTHAC: Maybe this online setting is already motivated?
%We believe the ability to update online is critical for RNNs, as many online settings are partially observable and RNNs are a typical solution approach.
Further, there is some evidence that this class of predictions is sufficient for a broad range of predictions about
the future
\citep{sutton2011horde,modayil2014multi,momennejad2018predicting,banino2018vector,white2015thesis,pezzulo2008coordinating}, and so the restriction to GVFs does not significantly limit representability. We therefore focus on developing an approach for this class of predictions within RNNs.
 %Further, these GVF features can be learned independent of span \citep{van2015learning} (i.e. with computation and history independent of the horizon), making it reasonable to believe that we can actually learn these short and long horizon predictions sufficiently quickly to be usable as state.

\section{GVF Networks} \label{GVFNs}

In this section, we introduce GVF Networks, an RNN architecture where hidden states are constrained to predict policy-contingent, multi-step outcomes about the future.
We first describe GVFs and the GVF Network (GVFN) architecture. In the following section, we develop the objective function and algorithms to learn GVFNs. There are several related predictive approaches, in particular TD Networks, that we discuss in Section \ref{sec_connections}, after introducing GVFNs.

We first need to extend the definition of GVFs \citep{sutton2011horde} to the partially observable setting, to use them within RNNs. The first step is to replace state with histories.
We define $\Hist$ to be the minimal set of histories, that enables the Markov property for the distribution over next observation
\begin{equation}
\!\Hist = \left\{ \hvec_t \!=\! (\obs_0, a_0, \ldots, \obs_{t-1}, a_{t-1}, \obs_t) \ | \ \substack{\text{(Markov property)} \Pr(\obs_{t+1} | \hvec_t, a_t ) = \Pr(\obs_{t+1} | \obs_{-1} a_{-1} \hvec_t a_t), \\ \text{ (Minimal history) }   \Pr(\obs_{t+1} | \hvec_t ) \neq \Pr(\obs_{t+1} | \obs_1, a_1, \ldots, a_{t-1}, \obs_t )} \right\}
\end{equation}
A GVF question is a tuple $(\tpolicy, \cumulant, \gamma)$ composed of a policy $\tpolicy: \Hist \times \Actions \rightarrow [0, \infty)$, cumulant
$\cumulant: \Hist \times \Actions \times \Hist \rightarrow \RR$ and continuation function\footnote{The original GVF definition assumed the continuation was only a function of $H_{t+1}$. This was later extended to transition-based continuation \citep{white2017unifying}, to better encompass episodic problems. Namely, it allows for different continuations based on the transition, such as if there is a sudden change from $\hvec_t$ to $\hvec_{t+1}$. We use this more general definition for this reason, and because the cumulant itself is already defined on the three tuple $(\hvec_t, a_t, \hvec_{t+1})$.} $\gamma: \Hist \times \Actions \times \Hist \rightarrow [0,1]$, also called the discount. On time step t, the agent is in $H_t$, takes actions $A_t$, transitions to $H_{t+1}$ and observes\footnote{Throughout this document, unbolded uppercase variables are random variables; lowercase variables are instances of that random variable; and bolded variables are vectors. When indexing into a vector on time step $t$, such as $\hvec_t$, we double subscript as $\hvec_{t,j}$ for the $j$th component of $\hvec_t$.} cumulant $C_{t+1}$ and continuation $\gamma_{t+1}$. The answer to a GVF question is defined as the value function, $V: \Hist \rightarrow \RR$, which gives the expected, cumulative discounted cumulant  from any history $\hvec_t \in \Hist$. The value function which can be defined recursively with a Bellman equation as
\begin{align}
  V(\hvec_t) &\defeq \expect*{ C_{t+1} + \gamma_{t+1} V(H_{t+1}) | H_t = \hvec_t, A_{t} \sim \pi(\cdot | \hvec_t)} \label{eq_bewh}\\
  &= \sum_{\action_t \in \Actions} \pi(\action_t | \hvec_t) \sum_{\hvec_{t+1} \in \Hists} \Pr(\hvec_{t+1} | \hvec_t, \action_t) \left[\cumulant(\hvec_t, a_t, \hvec_{t+1}) + \gamma(\hvec_t,a_t,\hvec_{t+1}) V(\hvec_{t+1}) \right] \nonumber
 .
\end{align}
The sums can be replaced with integrals if $\Actions$ or $\Observations$ are continuous sets. We assume that $\Hist$ is a finite set, for simplicity; the definitions and theory, however, can be extended to infinite and uncountable sets.

\begin{figure}
  \begin{center}
    \includegraphics[width=0.5\textwidth]{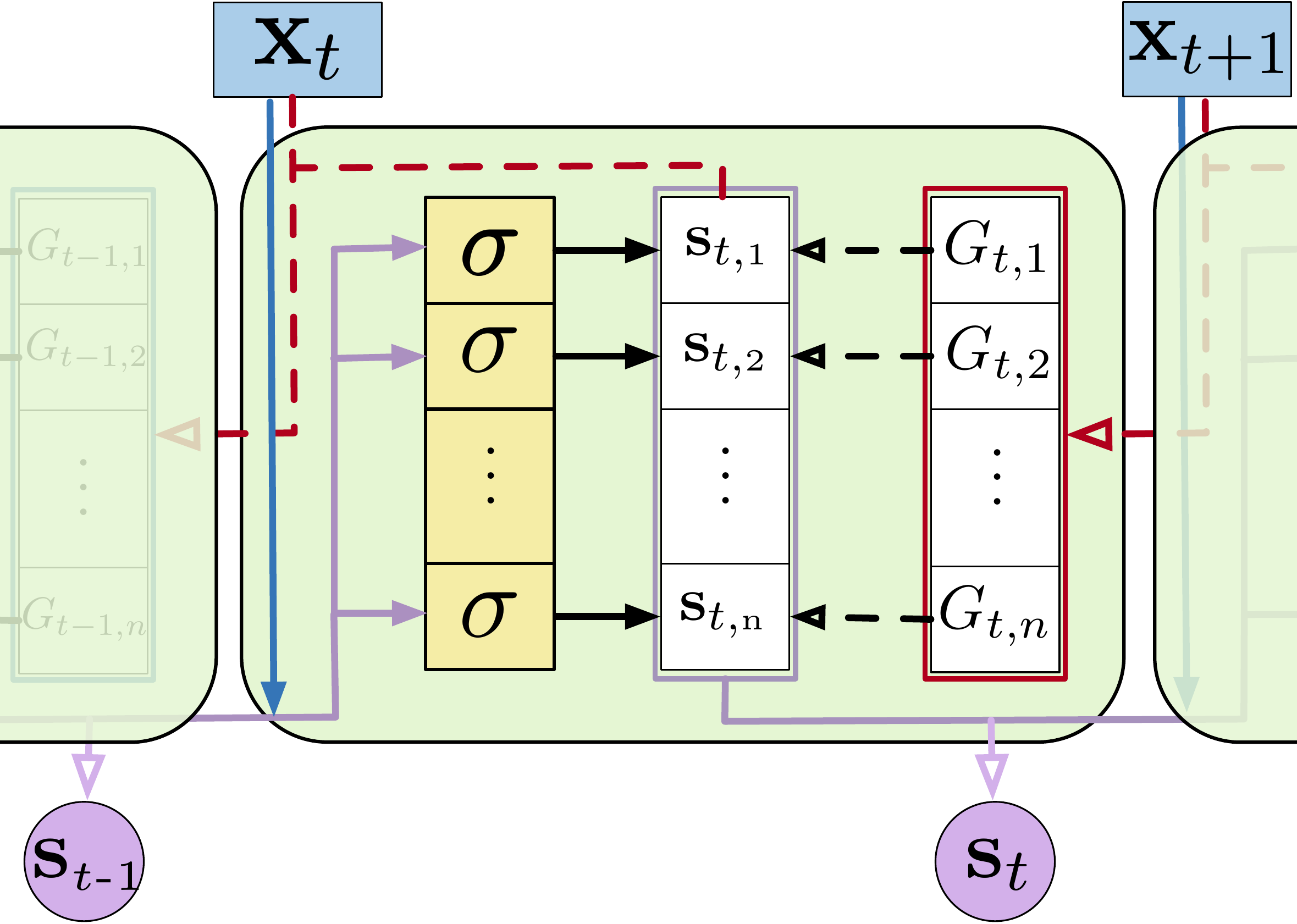}
  \end{center}
  \caption{GVF Networks (GVFNs), where each state component $\svec_{t,i}$ is updated towards the return $G_{t,i} \defeq C_{t+1}^{(i)} + \gamma_{t+1}^{(i)} \svec_{t+1,i}$ for the $i$th GVF. The solid forward arrows indicate how state is updated; in fact, the update is the same as a standard RNN. The difference is with the dotted lines, that indicate training. The dotted black arrows indicate the targets for the state components. The dotted red arrows indicate that the target $G_{t,i}$ are created using the observation and state on the next step.}\label{fig_gvfnsrnns}
\end{figure}

A GVFN is an RNN, and so is a state-update function $f$, but with the additional criteria that each element in $\svec_t$ corresponds to a prediction---to a GVF.
A GVFN is composed of $\numgvfs$ GVFs, with each hidden state component $\svec_{t,j}$ trained such that at time step $t$, $\svec_{t,j} \approx \vifunc{j}(\hvec_t)$ for the $j$th GVF and history $\hvec_t$. Each hidden state component, therefore, is a prediction about a multi-step policy-contingent question. The hidden state is updated recurrently as $\svec_t \defeq f_\weights(\svec_{t-1}, \xvec_t)$ for a parametrized function $f_\weights$, where $\xvec_t = [a_{t-1}, \obs_t]$ and $f_\weights$ is trained so that $\svec_j \approx \vifunc{j}(\hvec_t)$. This is summarized in Figure \ref{fig_gvfnsrnns}.

General value functions provide a rich language for encoding predictive knowledge. In their simplest form, GVFs with constant $\gamma$ correspond to multi-timescale predictions referred to as Nexting predictions \citep{modayil2014multi}. Allowing $\gamma$ to change as a function of state or history, GVF predictions can combine finite-horizon prediction with predictions that terminate when specific outcomes are observed \citep{modayil2014multi}.

\begin{figure}
  \centering
  \begin{subfigure}{0.43\textwidth}
    \includegraphics[width=0.8\textwidth]{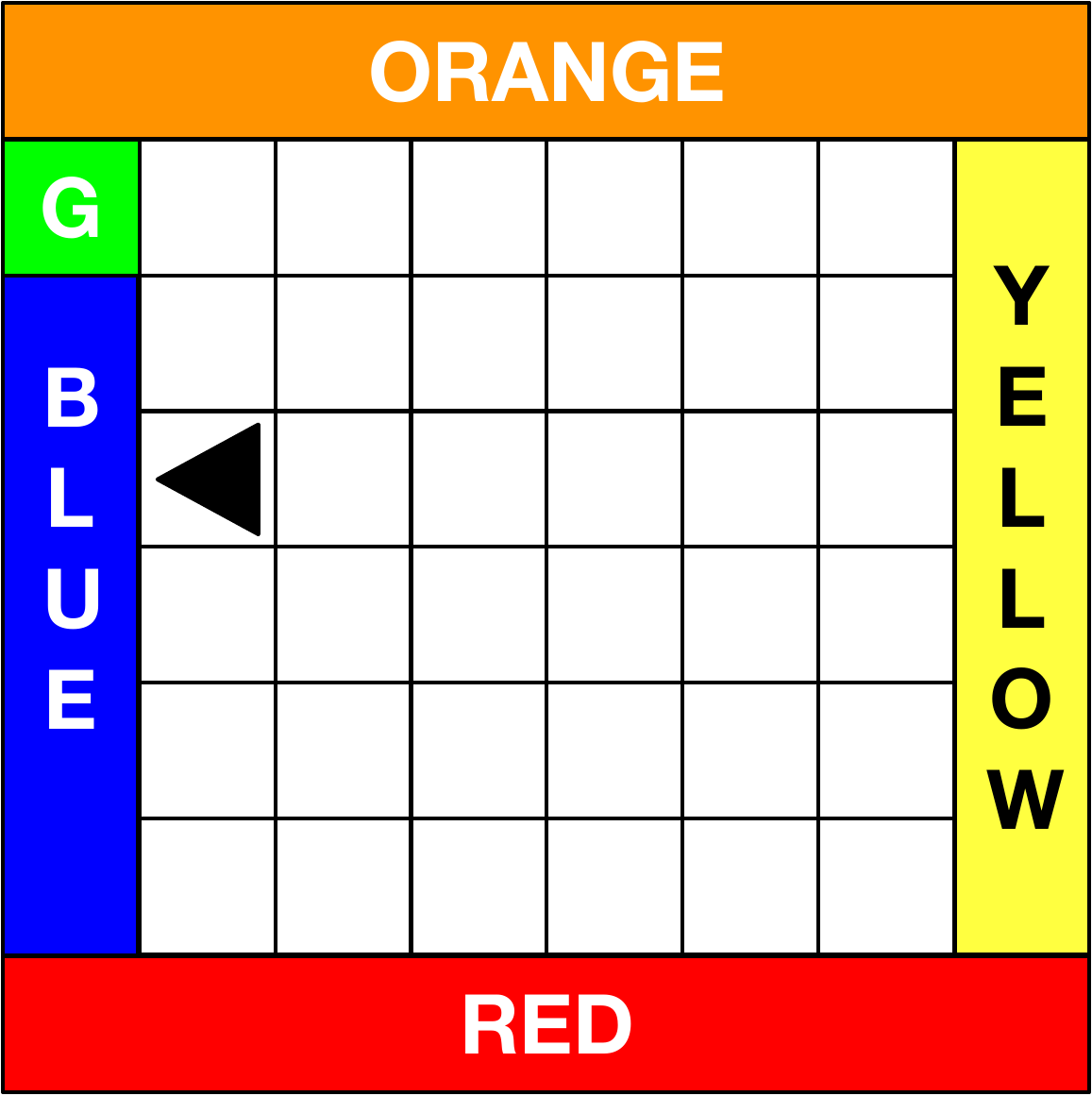}
  \end{subfigure}
  % \begin{minipage}{0.55\textwidth}
    \caption{ The Compass World: A partially observable grid world with observations of the color directly in front of the agent. \textbf{Actions:} The agent can take the actions Move Forward (one cell), Turn Left, and Turn Right. \textbf{Observations:} The agent observes the color of the grid cell it is facing. This means the agent can only observe a color if it is at the wall and facing outwards. The agent depicted as an arrow would see Blue. In the middle of the world, the agent sees White.  \textbf{Goal:} The agent's goal is to make accurate predictions about which direction it is facing. } \label{fig:compass_world_env}
  % \end{minipage}
\end{figure}

To build some intuition, we provide some examples in Compass World. This environment is used in our experiments and depicted in Figure \ref{fig:compass_world_env}. Compass World is a grid world where the agent is only provided information about the color directly in front it. This world is partially observable, with all the tiles in the middle having a white observation, with the only distinguishing color information available to the agent at the walls. The actions taken by the agent are to move forward, turn left, or turn right.

In this environment, the agent might want to know if it is facing the red wall. This can be specified as a GVF question: ``If I go forward until I hit a wall, what is the probability I will see red?". The policy is to always go forward. If the current observation is `Red', then the cumulant is 1; otherwise it is zero. The continuation $\gamma$ is 1 everywhere, except when the agent hits a wall and see a color; then it becomes zero. The sampled return from a state is 1.0 if the agent is facing the Red wall, because going forward will result in summing many zero plus a 1 right before termination. If the agent is not facing the Red wall, the return is 0, because the agent terminates when hitting the wall but only sees cumulants that are zero for the entire trajectory. Because the outcome is deterministic, the probabilities are 1 or 0.

% TODO: Adam, can you make this paragraph better?
The agent could also ask about how frequently it will see Red, within a horizon of about 10 steps. We can obtain an approximation to this question by using a constant continuation of $\gamma = 0.9$. The intuition for this comes from thinking of $1-\gamma$ as a success probability for a geometric distribution: the probability of successfully terminating. The mean of this geometric distribution is $\tfrac{1}{1-\gamma}$---which in this case is $\tfrac{1}{1-0.9}= 10$---provides the expected number of steps until the first success. Recall that termination indicates that a return is cut-off, and so a cumulant is not included in the sum after termination. This probabilistic termination means that even if Red is seen after 10 steps, it will still be included in the return. However, it does indicate its contribution has been significantly decayed. This exponential prediction loses precision, and so the GVF only provides an approximation to this question.
%For example, for a constant $\gamma$ of 0.99, with a cumulant of 1 when seeing `Red' and zero otherwise and a random policy, the GVF question is "How frequently will I see `Red' within the next 100 steps, if I drive around randomly?".

The agent could also also ask if it will see Red, within a horizon of about 10 steps. In this case, the continuation would be $0.9$ until the agent observed Red, at which point it would become zero (indicating termination). The GVF answer corresponds to a discounted probability of observing Red, with a smaller number if Red is observed further in the future. If the agent always see Red in 1 step from $\hvec_t$, then it observes $C_{t+1} = $ 1 and $\gamma_{t+1} = 0$ and the value is precisely 1. If the agent sees Red in 2 steps from $\hvec_t$, then $C_{t+1} = 0, \gamma_{t+1} = 0.9, C_{t+2} = 1$ and $\gamma_{t+2} = 0$ resulting in a value of $0.9$. If the agent sees Red in 10 steps from $\hvec_t$, then the value is $0.9^9 \approx 0.4$. If just a few more steps into the future, say 15 steps, then the value would be $0.2$. The magnitudes start to get quite low, indicating that it is less likely to observe Red in this window.

Notice that though we define the cumulants and continuation functions on the underlying (unknown) state $\hvec_t$, this is a generalization of defining it on the observations. The observations are a function of state; the cumulants and continuations $\gamma$ that are defined on observations are therefore defined on $\hvec_t$. In the examples above, these functions were defined using just the observations. More generally, we consider them as part of a problem definition. This means they could be defined using short histories, or other separate summaries of the history. As we discuss in Section \ref{sec_objective}, we can also consider cumulants that are a function of our own predictions or constructed state.

A natural question is how these GVFs are chosen. This problem corresponds to the discovery problem for predictive representations. In this work, we first focus on the utility of this architecture, with simple heuristics or expert chosen GVFs. We briefly discuss simple ideas for discovery in Section \ref{sec_discovery}, but leave a more systematic investigation of the discovery problem to future work.

\section{A Case Study using GVFNs for Time Series Prediction}\label{sec:case_study}
\begin{figure}[t]
  \centering
  % \includegraphics[width=0.9\textwidth]{plots/mg_example.pdf}
  % \begin{subfigure}{0.48\textwidth}
  %   \includegraphics[width=0.8\textwidth]{plot_jair/timeseries/raw/mg_tau_sens_train.pdf}
  % \end{subfigure}
  % \begin{subfigure}{0.48\textwidth}
  %   \includegraphics[width=\textwidth]{plot_jair/timeseries/mg_trunc_1.pdf}
  % \end{subfigure}
  % \begin{subfigure}{0.48\textwidth}
  %   \includegraphics[width=\textwidth]{plot_jair/timeseries/raw/mg_Predictions.pdf}
  % \end{subfigure}
  \includegraphics[width=0.95\textwidth]{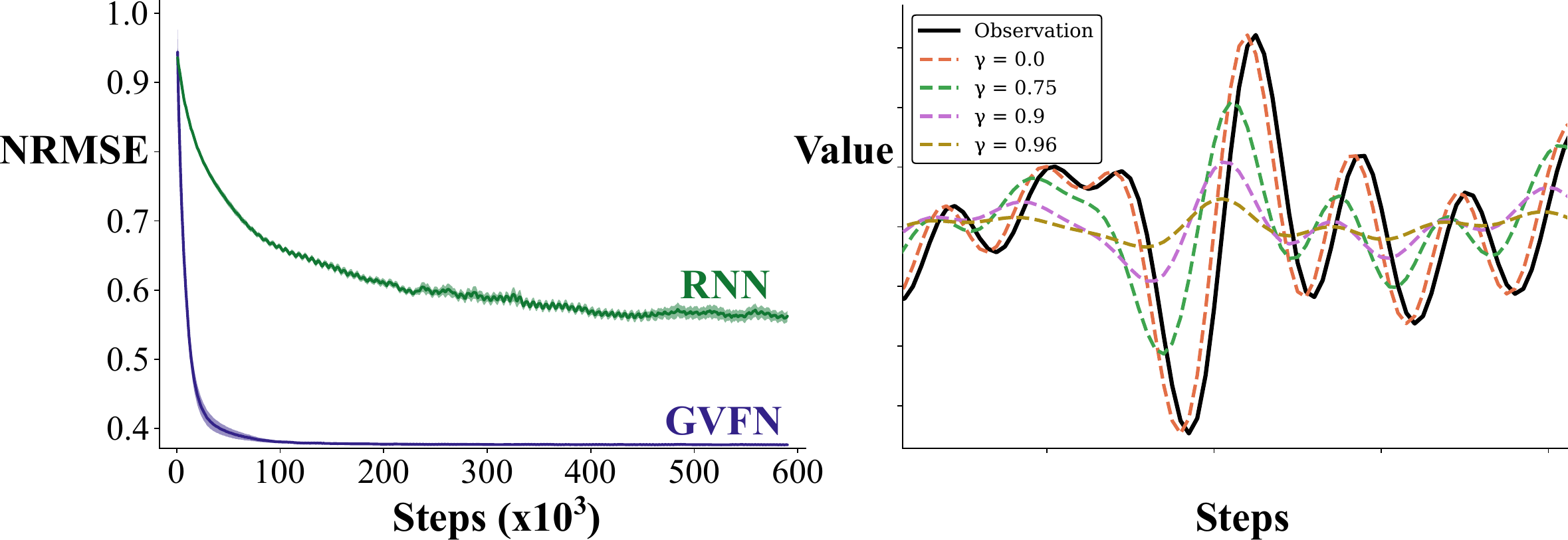}
  \caption{\textbf{(left)} Example learning curve for MSO with GVFNs and simple RNNs \textbf{(right)} The returns for different $\gamma$, corresponding to GVFs in the GVFN, for a small section of the MSO time series dataset. The dotted red line for $\gamma = 0$ looks overlayed with the time series plotted in black, but is actually the observation one step in the future.
  }\label{fig:mg_example}
\end{figure}

Before discussing the objective and training algorithms for GVFNs, we provide a simple demonstration of their use in a synthetic single-variable time series dataset to build intuition. GVFNs can be used for time series prediction by simply assuming that a fixed (unknown) policy generates the data. The GVFs within the network are assumed to have this same fixed unknown policy in common, but differ in the pair of continuation and cumulant functions. For a multi-variate time series, one GVF could have a cumulant corresponding to the first entry of the observation on the next time step, and another GVF could use the second entry. Even for a single-variate time series, we can define meaningfully different cumulants. For example, one GVF could correspond to the probability that the observation becomes larger than 1. The cumulant would be zero until this event occurs, at which point it would be 1.

 In Figure \ref{fig:mg_example} (left) we provide a preliminary result using a GVFN to forecast 12-steps into the future on the single-variate Multiple Super-imposed Oscillator (MSO) time series dataset.
 We discuss the full empirical set-up in Section \ref{sec:exp_forecasting}, and here simply provide some insights relevant to building intuition for how to use GVFNs.

The GVFN consists of a recurrent, constrained layer of 128 GVFS with $\gamma$s spaced linearly in $[0.1,0.97]$ to learn the state. To make predictions, we can additionally add feedforward layers from this recurrent layer; here we add a ReLu layer for additional nonlinearity in the prediction. For comparison, we also include a simple RNN, which similarily uses an additional ReLu layer after its recurrent layer. The prediction target is the observation 12 steps into the future. Both the RNN and GVFN have to wait 12 steps to see the accuracy of their prediction, delaying updates based on the target by 12 steps. The GVFN, however, can use the loss on the state at each step, and so more directly influence the value of states with the most recent observations. Both methods use p-BPTT, with truncation $p$. With a sufficiently high $p$, both perform well (see Section \ref{sec:exp_forecasting} for results with many $p$). We report the result here for $p = 1$, where the GVFN already obtains near-optimal performance.

It might be surprising that this simple GVFN, with GVFs only differing in continuation $\gamma$, can perform well. For time series data, however, such constant $\gamma$ predictions provide anticipatory information about observations in the future.
%To see why,  we plot the true returns corresponding to the GVFN's hidden state, in the right plot of Figure \ref{fig:mg_example}. The returns increase several steps ahead of these increases in the time series itself, because the accumulate future observation values. The return provide useful predictive information about increases and decreases that will soon appear in the time series.
To see why, we plot the time series as well as returns for $\gamma\in\{0, 0.75,0.9, 0.96\}$ as dotted lines in Figure \ref{fig:mg_example} (right). These returns reflect the type of information that would be provided by a GVF prediction. At each time point $t$ on the x-axis, we can see that the smaller $\gamma$, like $\gamma = 0.75$ as dotted green, anticipate the observations in a nearby window. If the time series is starting to rise in the near future, then the dotted green starts to rise right now. Returns can thus provide useful predictive information about increases and decreases that are expected to soon appear in the time series.
Notice that the magnitude of the returns are approximately equal. For practical use, we want the magnitude of each GVF prediction to be similar, to avoid large differences in magnitude between state variables. With large $\gamma$, however, the return becomes large and so too does the value function. The standard fix to this is straightforward: each GVF uses a scaled cumulant of $(1-\gamma) o_{t+1}$.

% In our experiment, we use one of the simplest choices: the GVFs all have the same cumulant---the observation on the next step---with different constant continuations $\gamma$. The primary difference between GVFs is how much they discount future observations. The true returns are $\sum_{i = 0}^\infty \gamma^i o_{t+1+i}$. Changing the continuation results in differences in the horizon for the GVF. For a small $\gamma$, the GVF is effectively only making predictions about observations in the near future. At this extreme, for $\gamma =0$, the GVF is exactly making a prediction about the immediate observation, because $\sum_{i = 0}^\infty 0^i o_{t+1+i} = o_{t+1}$. You can also see how $\gamma = 0$ multiplies the value in the next state in Equation \eqref{eq_bewh}. The value function then simply equals the expected cumulant, conditioned on history, which here is the expected next observation.

% As $\gamma$ becomes larger, the weight $\gamma^i$ in the return remains big for observations further in the future. As mentioned above, the effective horizon is $(1-\gamma)^\inv$, which corresponds to an expected number of steps until termination. For $\gamma = 0.99$, the effective horizon is 100 and the corresponding observation is weighted by $0.99^{100} = 0.36$. This means that we still have a non-negligble weight on the 100th step, but in expectation, we approximately end the discounted accumulation of cumulants at 100 steps.

Notice, though, that there is a trade-off between anticipating a cumulant farther into the future and the precision of predictions about the future. Returns with lower continuations predict trends closer to when they occur in the dataset and have higher resolution. Returns with higher continuations anticipate changes further in the future, at the cost of smoothing over the detailed changes in the dataset. By using both lower and higher continuations, we hope to obtain the benefits of both. We further discuss this simple heuristic---GVFs with the same cumulant and varying $\gamma$---as a general purpose heuristic in Section \ref{sec_discovery}.

\section{The Objective for GVFNs}\label{sec_objective}

In this section, we introduce the objective function for GVFNs, that constrains the learned state to be GVF predictions.
Each state component of a GVFN is a value function prediction, and so is approximating the fixed point to a Bellman equation with history in Equation \eqref{eq_bewh}. The extension is not as simple as using a standard Bellman operator, however, because the GVFs are in a network. In fact, the Bellman equations are coupled in two ways: through composition---where one GVF can be the cumulant for another GVF as seen in section \ref{section:experiments:poorlyspecified}---and through the parametric recurrent state representation. We first discuss the Bellman network operator in Section \ref{sec_operator}, which extends the typical Bellman operator to allow for composition. We then explain how the coupling that arises from the recurrent state representation can be handled using a projected operator, and provide the objective for GVFNs, called the Mean-Squared Projected Bellman Network Error (MSPBNE), in Section \ref{sec_obj}. Then we discuss several algorithms to optimize this objective in Section \ref{sec_algs}.

The GVFN objective we introduce can be added to the standard RNN objective, to provide an RNN where the learned states are both useful for prediction of the target and encouraged---or regularized---to be GVF predictions. In this work, we only train GVFNs with the GVFN objective, without including the loss to a target, to focus the investigation on the utility of the proposed objective and on predictive features.

\subsection{The Bellman Network Operator} \label{sec_operator}

To understand the Bellman network operator, it is useful to first revisit the Bellman operator for learning a single GVF.
We assume the set of histories $\Hist$ is finite.
%\footnote{It is common to assume finite state spaces when analyzing value functions and defining Bellman operators. Extensions to infinite spaces is possible, but complicates presentation.}
Assume a tabular encoding for the values, $\vi{j} \in \RR^{|\Hists|}$, for a GVF question $(\pij{j}, \cj{j}, \gammaj{j})$. The Bellman equation in \ref{eq_bewh} can be written as a fixed point equation, with Bellman operator
\begin{equation}
\Bn^{(j)} \vi{j} \defeq \Cpij{j} + \Ppigammaj{j} \vi{j}
\end{equation}
where $\Cpij{j} \in \RR^{|\Hists|}$ is the vector of expected cumulant values under $\pij{j}$, with entries
\begin{equation}
\Cpij{j}(\hvec_t) \defeq \sum_{a_t \in \Actions} \pij{j}(a_t | \hvec_t) \sum_{\hvec_{t+1} \in \Hists} \Pr(\hvec_{t+1} | \hvec_t, a_t) \cj{j}(\hvec_t, a_t, \hvec_{t+1})
.
\end{equation}
and
 $\Ppigammaj{j} \in \RR^{|\Hists| \times |\Hists| }$ is the matrix of values satisfying
\begin{equation}
\Ppigammaj{j}(\hvec_t, \hvec_{t+1}) = \sum_{a_t \in \Actions} \pij{j}(a_t | \hvec_t) \Pr(\hvec_{t+1} | \hvec_t, a_t) \gammaj{j}(\hvec_t, a_t, \hvec_{t+1})
.
\end{equation}
If the operator $\Bn^{(j)}$ is a contraction, then iteratively applying this operator converges to a fixed point. More precisely, if for any $\vi{j}_1, \vi{j}_{2} \in \RR$, $\| \Bn^{(j)} \vi{j}_1 -  \Bn^{(j)}\vi{j}_2 \| <  \| \vi{j}_1 - \vi{j}_2 \| $, then iteratively applying $\Bn^{(j)}$, as $\vi{j}_2 = \Bn^{(j)} \vi{j}_1, \ldots, \vi{j}_{t+1} = \Bn^{(j)} \vi{j}_t, \ldots$, converges to a fixed point.  Because temporal difference learning algorithms are based on this fixed-point update, the Bellman operator is central to the analysis of many algorithms for learning value functions, and is used in the definition of objectives for value estimation.

We can similarly define a Bellman operator that accounts for the relationships between GVFs in the network. Assume there are $\numgvfs$ GVFs, with $\vinone \in \RR^{\numgvfs | \Hists |}$ the stacked values for all the GVFs,
\begin{equation}
\vinone \defeq \left[\begin{array}{c}
\vi{1}\\
\vdots \\
\vi{\numgvfs}
\end{array}
\right]
.
\end{equation}
The cumulants may now be functions of the values of other GVFs; we therefore explicitly write $\Cpij{j}_{\vinone}$.
The Bellman network operator $\Bn$ is
%
% \todo[inline]{Should it be $C^{(1)}_{\vinone} + \Ppigamma{1} \vi{1}$?}
%
\begin{equation}
\Bn \vinone
\defeq
\left[\begin{array}{c}
\Cpij{1}_{\vinone} + \Ppigammaj{1} \vi{1}\\
\vdots \\
\Cpij{\numgvfs}_{\vinone} + \Ppigammaj{\numgvfs} \vi{\numgvfs}
\end{array}
\right]
.
\end{equation}
The Bellman network operator needs to be treated as a joint operator on all the GVFs because of compositional predictions, where the prediction on the next step of GVF $j$ is the cumulant for GVF $i$. When iterating the Bellman operator $\vi{j}$ is not only involved in its own Bellman equation, but also in the Bellman equation for $\vi{i}$. Notice that if there were no compositions, the Bellman network operator would separate into individual Bellman operators, that operate on each $\vi{j}$ independently.

To use such a Bellman network operator, we need to ensure that iterating under this operator converges to a fixed point. For no composition, this result is straightforward, as it simply follows from previous results showing when the Bellman operator is a contraction. We state this explicitly below in Corollary \ref{cor_main}. Under composition, we need to consider the effect of the current value function on the cumulant. Consequently, the operator may no longer be a simple linear projection of the values, followed by a sum of expected cumulants.

We first identify a necessary condition: the connections between GVFs must be acyclic. For example, GVF $i$ cannot be a cumulant for GVF $j$, if $j$ is already a cumulant for $i$. More generally, the connections between GVFs cannot create a cycle, such as $1 \rightarrow 2 \rightarrow 3 \rightarrow 1$. We provide a counterexample, where the Bellman network operator is not a contraction when there is a cycle, to illustrate that this condition is necessary.

We further place restrictions on the cumulant, if it is a function of other GVFs. In particular, we require that the cumulant is a Lipschitz function of the other value functions.
Note that this restriction encompasses the setting for a non-compositional GVF, because the cumulant can be a constant w.r.t. these values. It also encompasses the setting we use in our experiments: that each cumulant is a linear function of the GVF values on the next step.

\begin{assumption}[Acyclic Connections]
The directed graph $G$ is acyclic. $G$ consists of $\numgvfs$ vertices, each corresponding to a GVF, and each directed edge $(i,j)$ indicates that $j$ is used in the cumulant for $i$.
\end{assumption}
\begin{assumption}[Lipschitz Compositional Cumulants]
If GVF $i$ has directed edges to $\{j_1, \ldots, j_k\}$, then the cumulant $c^{(i)}_{\vinone}(\hvec_{t+1})$ is Lipschitz in $\vi{j_1}, \ldots, \vi{j_k}$ with Lipschitz constant $K_i$. That is, for $\vinone_1,\vinone_2 \in \RR^{\numgvfs | \Hists |}$, $\| \Cpij{i}_{\vinone_1} - \Cpij{i}_{\vinone_2} \| \le K_i \sum_{l=1}^k \|  \vi{j_l}_1 - \vi{j_l}_2 \|$.
\end{assumption}
Note that this assumption is satisfied if we assume that for some bounded weights $w_1, \ldots, w_k \in \RR$, the cumulant must satisfy $c^{(i)}_{\vinone}(\hvec_{t+1}) = \sum_{l=1}^k w_l \vi{j_l}(\hvec_{t+1})$ or equivalently, $\Cpij{i}_{\vinone} = \sum_{l=1}^k w_l \Ppigammaj{j_l} \vi{j_l}$.  This is because $\Ppigammaj{j_l}$ is a non-expansion, and so
\begin{align*}
\| \Cpij{i}_{\vinone_1} - \Cpij{i}_{\vinone_2} \|
= \left\| \sum_{l=1}^k w_l \Ppigammaj{j_l} (\vi{j_l}_1 - \vi{j_l}_2) \right\|
 &\le \sum_{l=1}^k | w_l | \| \Ppigammaj{j_l} (\vi{j_l}_1 - \vi{j_l}_2) \|\\
& \le (\max_{l} | w_l | ) \sum_{l=1}^k \| \vi{j_l}_1 - \vi{j_l}_2 \|
.
\end{align*}

The third assumption is standard for showing Bellman operators are contractions, and is easily satisfied if the policy is proper: is guaranteed to visit at least one state where the continuation is less than 1.
\begin{assumption}[Discounted Transitions are Contractions]
For all $j \in \{1, \ldots, \numgvfs\}$, $\beta_j \defeq \| \Ppigammaj{j} \| < 1$, where $\| \cdot \|$ is the spectral norm.
\end{assumption}
With these three assumptions, we can prove the main result.
\begin{theorem}\label{thm_main}
Under Assumptions 1-3, iterating $\vt{t+1} = \Bn \vt{t}$ converges to a unique fixed point.
\end{theorem}
\begin{proof}
We first prove that the sequence of value estimates converges (Part 1) and then that it converges to a unique fixed point (Part 2 and 3).

\noindent
\textbf{Part 1:} \textit{The sequence $\vinone_{1}, \vinone_{2}, \ldots$ defined by $\vinone_{t+1} = \Bn \vinone_t$ converges to a limit $\vinone^* \in \RR^{\numgvfs|\Hists|}$.}

Because $G$ is acyclic, we have a linear topological ordering of the vertices, $i_1, \ldots, i_\numgvfs$: for each directed edge $(i,j)$, $i$ comes before $j$ in the ordering. Therefore, starting from the last GVF $j = i_\numgvfs$, we know that the Bellman operator $\Bn^{(j)}$ is a contraction with rate $\beta_{j} < 1$,
\begin{equation*}
\| \Bn^{(j)} \vit{j}{1} - \Bn^{(j)} \vit{j}{0} \| = \| \Ppigammaj{j} \vit{j}{1} - \Ppigammaj{j}  \vit{j}{0} \| \le \beta_j\| \vit{j}{1} - \vit{j}{0} \|
.
\end{equation*}
Therefore, iterating $\Bn$ for $t$ steps results in the error
\begin{equation*}
\| \vit{j}{t+1} - \vit{j}{t} \| \le \beta_j^t \| \vit{j}{1} - \vit{j}{0} \|
\end{equation*}
and as $t \rightarrow \infty$, $\vit{j}{t}$ converges to its fixed point.

We will use induction for the argument, with the above as the base case.
Assume for all $j \in \{i_k, \ldots, i_{\numgvfs}\}$ there exists a ball of radius $\epsilon(t)$ where $\| \vit{j}{t+1} - \vit{j}{t} \| \le \epsilon(t)$ and $\epsilon(t) \rightarrow 0$ as $t \rightarrow \infty$.
Consider the next GVF in the ordering, $i = i_{k-1}$.

\textbf{Case 1: } There are no outgoing edges from $i$. If $i$ does not use another GVF $j$ in its cumulant, then iterating with $\Bn$ independently iterates $\vit{i}{t}$ with $\Bn^{(i)}$. Therefore, as above, $\vit{i}{t}$ converges because the Bellman operator is a contraction. In this setting, clearly such an $\epsilon_i(t)$ exists because $\| \vit{j}{t+1} - \vit{j}{t} \| \rightarrow 0$ as $t \rightarrow \infty$.

\textbf{Case 2: } The cumulant for GVF $i$ is composed of the values for the set of GVFs $\mathcal{J} \subseteq  \{i_k, \ldots, i_{\numgvfs}\}$. The basic idea, formalized below, is that GVF $i$ will be guaranteed to converge once the GVFs used to construct the become sufficiently accurate. The update is $\vit{i}{t+1} =  \Cpij{i}_{\vinone_t}  + \Ppigammaj{i} \vit{i}{t}$. The change in $\vit{i}{t}$ is
\begin{align*}
\| \vit{i}{t+1} - \vit{i}{t} \|
&=  \| (\Cpij{i}_{\vinone_t} - \Cpij{i}_{\vinone_{t-1}})+ \Ppigammaj{i} ( \vit{i}{t} - \vit{i}{t-1}) \|\\
&\le K_i \sum_{j \in \mathcal{J}} \| \vit{j}{t} - \vit{j}{t-1}\| + \beta_i \| \vit{i}{t} - \vit{i}{t-1} \|\\
&\le  \numgvfs K_i  \epsilon(t-1) + \beta_i \| \vit{i}{t} - \vit{i}{t-1} \|
.
\end{align*}
In the first inequality, the first term is due to Lipschitz continuity of the cumulant and the second term is due to the fact that $\| \Ppigammaj{i}  \| = \beta_i$. In the second inequality, we know $\| \vit{j}{t} - \vit{j}{t-1}\| \le \epsilon_j(t)$, under the inductive hypothesis. The second inequality is loose, as the sum only involves $|\mathcal{J}| < \numgvfs$ terms, but we use $\numgvfs$ for simplicity since the results goes through with this constant as well.
For sufficiently large $t$, $\epsilon(t-1)$ can be made arbitrarily small.
If  $\numgvfs K_i \epsilon(t-1) < (1-\beta_i) \| \vit{i}{t} - \vit{i}{t-1} \|$, i.e., $\epsilon(t-1) < \tfrac{(1-\beta_i)}{\numgvfs K_i} \| \vit{i}{t} - \vit{i}{t-1} \|$ then
\begin{align*}
\| \vit{i}{t+1} - \vit{i}{t} \|
&\le \tilde{\beta}_i \| \vit{i}{t} - \vit{i}{t-1} \| \hspace{1.0cm}\text{for some $\tilde{\beta}_i < 1$}
\end{align*}
and so the iteration is a contraction on step $t$.
Else, if $\epsilon(t-1) \ge \tfrac{(1-\beta_i)}{\numgvfs K_i} \| \vit{i}{t} - \vit{i}{t-1} \|$, then this implies the difference $\| \vit{i}{t+1} - \vit{i}{t} \|$ is already within a small ball, with radius $\numgvfs K_i \epsilon(t-1)/(1-\beta_i)$.  As $t \rightarrow \infty$, the difference can oscillate between being within this ball, which shrinks to zero because $\epsilon(t)$ shrinks to zero, or being iterated with a contraction that also shrinks the difference. In either case, there exists an $\epsilon_i(t)$ such that  $\| \vit{i}{t+1} - \vit{i}{t} \| \le \epsilon_i(t)$, where $\epsilon_i(t) \rightarrow 0$ as $t \rightarrow \infty$.

By induction, we have such an $\epsilon_i$ for all GVFs in the network. Therefore, we know the sequence $\vit{i}{t}$ converges.

\noindent
\textbf{Part 2:} \textit{$\vinone^*$ is a fixed point of $\Bn$.}

Because the Bellman network operator is continuous, the limit can be taken inside the operator
\begin{equation*}
\vinone^* = \lim_{t \rightarrow \infty} \vt{t}
= \lim_{t \rightarrow \infty} \Bn\vt{t-1}
= \Bn \left(\lim_{t \rightarrow \infty} \vt{t-1}\right) = \Bn \vinone^*
\end{equation*}

\noindent
\textbf{Part 3: } \textit{$\vinone^*$ is the only fixed point of $\Bn$.}

Consider an alternative solution $\vinone$. Because of the uniqueness of fixed points under Bellman operators, all those GVFs that have non-compositional cumulants have unique fixed points and so those components in $\vinone$ must be the same as $\vinone^*$. All the GVFs next in the ordering that use those GVFs as cumulants must then also converge to a unique value, because their Bellman operators with fixed GVFs as cumulants have a unique fixed point. This argument continues for the remaining GVFs in the ordering.
\end{proof}

\begin{corollary}\label{cor_main}
Under Assumption 3 with non-compositional cumulants (no edges in $G$), iterating $\vt{t+1} = \Bn \vt{t}$ converges to a unique fixed point.
\end{corollary}

\begin{proposition}[Necessity of Acyclic Composition]\label{prop_counter}
There exists transition function $\Pfcn: \States \times \Actions \times \States \rightarrow [0,1]$ and policy $\pi: \States \times \Actions \rightarrow [0,1]$ such that, for two GVFs in a cycle, iteration with the Bellman network operator diverges.
\end{proposition}
\begin{proof}
Assume there are two states, with the policy defined such that we get the following dynamics for the Markov chain
\begin{equation}
\Ppi =
\left[\begin{array}{cc}
0.9 & 0.1\\
0.1 & 0.9
\end{array}
\right]
.
\end{equation}
Assume further that $\gamma = 0.95$. The resulting Bellman iteration is
\begin{align*}
\twovec{\vi{1}}{\vi{2}}
&= \Ppi \twovec{\vi{2}}{\vi{1}} + \gamma  \Ppi \twovec{\vi{1}}{\vi{2}} \\
&= \Ppi \left[\begin{array}{cc}
0 & 1\\
1 & 0
\end{array}
\right] \twovec{\vi{1}}{\vi{2}} + \Ppi \left[\begin{array}{cc}
\gamma & 0\\
0 & \gamma
\end{array}
\right] \twovec{\vi{1}}{\vi{2}} \\
&= \Ppi \left[\begin{array}{cc}
\gamma & 1\\
1 & \gamma
\end{array}
\right] \twovec{\vi{1}}{\vi{2}}
\end{align*}
Since the matrix $\Ppi \left[\begin{array}{cc}
\gamma & 1\\
1 & \gamma
\end{array}
\right] $
is an expansion, for many initial $\twovec{\vi{1}}{\vi{2}}$ this iteration goes to infinity, such as initial $\vi{1} = \vi{2} = \twovec{1}{1}$.
\end{proof}

\subsection{The Objective Function for GVFNs} \label{sec_obj}

With a valid Bellman network operator, we can proceed to defining the objective function for GVFNs. The above fixed point equation assumes a tabular setting, where the values can be estimated directly for each history. GVFNs, however, have a restricted functional form, where the value estimates must be a parametrized function of the current observation and value predictions from the last time step. Under such a functional form, it is unlikely that we can exactly solve for the fixed point. Rather, we will solve for a projected fixed point, which projects into the space of representable value functions.

Define the space of functions as
\begin{align}
\mathcal{F} = \Big\{ &\vinone_\weights = [\vi{1}_\weights, ..., \vi{\numgvfs}_\weights] \in \RR^{\numgvfs|\Hists|}  \ \ | \ \ \text{ where } \weights \in \weightspace \ \ \text{ and } \\
&V_\weights(\hvec_{t+1}) = f_\weights([\viweights{1}(\hvec_t), \ldots, \viweights{\numgvfs}(\hvec_t)], \xvec_{t+1})
\ \ \text{ when } \text{Pr}(\hvec_{t+1} | \hvec_t, \xvec_{t+1}) > 0 \Big\} \nonumber
\end{align}
Recall that $\xvec_{t+1} = [a_t, \obs_{t+1}]$. We know $\text{Pr}(\hvec_{t+1} | \hvec_t, \xvec_{t+1}) > 0$ only when $\hvec_{t+1} \equiv \hvec_t a_t \obs_{t+1}$, and so expect this to only be true for one outcome $\hvec_{t+1}$. We write that $\hvec_{t+1}$ is equivalent, rather than equal, to the current history appended with action $a_t$ and observation $\obs_{t+1}$, because $\hvec_{t+1}$ might be shorter (more minimal): earlier actions and observations might not be needed.
Define the projection operator
\begin{align}
\Pi_{\mathcal{F}}(\vinone) &\defeq \min_{\hat{\vinone} \in \mathcal{F}} \| \vinone - \hat{\vinone} \|_{\dw}^2
\hspace{0.5cm}\text{ where } \| \vinone - \hat{\vinone} \|_{\dw}^2 \defeq \sum_{\hvec \in \Hist} \dw(\hvec) (V(\hvec) - \hat{V}(\hvec))^2
\end{align}
where $\dw: \Hists \rightarrow [0,1]$ is the sampling distribution over histories. Typically, we assume data is generated by following a behavior policy $\mu: \Hists \rightarrow [0,1]$, and that $\dw$ is the stationary distribution for this policy. The value functions for policies $\pi_i$ are typically learned off-policy, since in general $\pi_i$ will not equal $\mu$. The behavior policy $\mu$ used to gather the data is different, or off of, the policy---or policies---that we are evaluating.

To obtain the projected fixed point solution, a natural goal is to minimize the following projected objective,
\begin{equation}
\min_{\weights \in \weightspace} \| \Pi_{\mathcal{F}} \Bn \vinone_\weights - \vinone_\weights \|_{\dw}^2
\end{equation}
Unfortunately, this objective can be hard to compute, because the projection operator $\Pi_{\mathcal{F}}$ onto the nonlinear manifold can be intractable. Instead, we take the same approach as \citeA{maei2010toward}, when defining the nonlinear MSPBE for learning value functions with neural networks and other nonlinear function approximators. The idea is to approximate the projection onto the nonlinear manifold by assuming it is locally linear. Then, we can use a linear projection operator, defined locally at the current set of parameters $\weights \in \weightspace$, spanned by the basis $\phivec_{j,\weights}(\hvec) \defeq \nabla_\theta \vi{j}_\weights(\hvec)$ for all $\hvec \in \Hists$ and GVFs $j$. Let $\phimat_{j,\weights}$ correspond to the matrix of stacked $\phivec_{j,\weights}(\hvec)^\trans$ for all $\hvec \in \Hists$, having $|\Hists|$ rows.  We further define
\begin{align*}
  \phimat_{\weights}
  \defeq
  \left[\begin{array}{c}
          \phimat_{1, \weights}\\
          \vdots \\
          \phimat_{\numgvfs, \weights}
        \end{array}\right]
  \quad
  \quad
  \quad
  \dwdiag \defeq \diag\left[
  \begin{array}{c}
    \dw \\
    \vdots \\
    \dw
  \end{array}
  \right]
  \quad
  \quad
  \quad
  \Pi_{\weights}
  \defeq
    \phimat_{\weights}
    (\phimat_{\weights}^\trans \dwdiag \phimat_{\weights})^\inv
    \phimat_{\weights}^\trans \dwdiag
    .
\end{align*}
Using this locally linear approximation to the objective potentially expands the set of stationary points. The fixed points under the original projection are still fixed points under this locally linear approximation. But, there could be points that are fixed points under this locally linear approximation, that would not be under the original.

We call the final objective using this projection the MSPBNE\footnote{A variant of the MSPBNE has been introduced for TD networks \citep{silver2012gradient}; the above generalizes that MSPBNE to GVF Networks. Because it is a strict generalization, we use the same name.}, defined as
\begin{align}
    \text{MSPBNE}(\weights) &\defeq \| \Pi_{\weights} \Bn \vinone_\weights - \vinone_\weights \|_{\dw}^2 \label{eq_projform}
   \end{align}
 We show in the following lemma, with proof in Appendix \ref{appendix:mspbne}, that in can be rewritten in a way that makes it more amenable to compute and sample gradients.\footnote{Since developing the MSPBNE, an alternative approach to defining a nonlinear MSPBE has been developed using a conjugate form for the Bellman error (see \citeR{dai2017learning} and in-preparation work that makes the connection the MSPBE more explicit \citep{patterson2020investigating}). The extension here should be relatively straightforward, as we formulate the objective using histories.} We will use this reformulation to develop algorithms to minimize this objective in the next section.
\begin{restatable}{lemma}{mspbnelemma}\label{lemma:mspbne_exp}
The MSPBNE defined in Equation \eqref{eq_projform} can be rewritten as
   \begin{align}
    \text{MSPBNE}(\weights) &= \boldsymbol{\delta}(\weights)^\top W(\weights)^\inv  \boldsymbol{\delta}(\weights) \label{eq_mspbne}
   \end{align}
   where
    \vspace{-0.5cm}
   \begin{align}
     W(\weights) &\defeq
       \Expected_d\bigg[\sum_{j=1}^\numgvfs \phivec_{j,\weights}(H) \phivec_{j,\weights}(H)^\trans \bigg]
       = \sum_{\hvec \in \Hists} d(\hvec) \sum_{j=1}^\numgvfs \phivec_{j,\weights}(\hvec) \phivec_{j,\weights}(\hvec)^\trans \label{eqn_w}\\
     \boldsymbol{\delta}(\weights) &\defeq
     \sum_{j=1}^\numgvfs \Expected_{d,\pi_j}\bigg[\tderror_j(H, A, H') \phivec_{j,\weights}(H) \bigg] \nonumber\\
 %MARTHAC: Maybe this doesnt help in the main body, since the appendix gives these details    &= \sum_{j=1}^\numgvfs \sum_{\hvec \in \Hists} d(\hvec) \sum_{a \in \Actions} \pi_j(a|\hvec) \Expected\bigg[\tderror_j(\hvec, a, H') \phivec_{j,\weights}(\hvec) | H = h, A = a \bigg] \nonumber \\
     \tderror_j(H,A,H') &\defeq c^{(j)}(H, A, H') + \gamma^{(j)}(H, A, H')\viweights{j}(H') - \viweights{j}(H) \nonumber
     .
   \end{align}
 \end{restatable}
From this reformulation, one can see that the MSPBNE objective is a weighted quadratic objective, with weighting matrix $W(\weights)$ on vector $\boldsymbol{\delta}(\weights)$. The objective is zero---and so minimal---when $\boldsymbol{\delta}(\weights) = \zerovec$. This is similar to the temporal difference (TD) learning fixed point criteria. In fact, TD implicitly optimizes the linear MSPBE, which corresponds to the above objective with $\numgvfs = 1$ and fixed features that do not depend on the parameters. Once we have a projected Bellman error objective, we can take advantage of the many advances in formulating TD algorithms to optimized MSPBE objectives. Therefore, though this objective looks quite complex, there is substantial literature to facilitate minimizing the MSPBNE.

\section{Algorithms for the MSPBNE} \label{sec_algs}

The algorithms to optimize the MSPBNE are a relatively straightforward combination of standard algorithms for RNNs and the TD algorithms designed to optimize the MSPBE. To provide some intuition on these algorithms, and how to obtain this combination of TD and RNN algorithms, we begin with a simpler setting: extending TD to a recurrent setting,  with one GVF. From there, we introduce two algorithms for the MSPBNE: Recurrent TD and Recurrent GTD.

Consider first the on-policy TD update, without recurrence, assuming the true state $\svec_t$ at time t is given:
   \begin{align*}
     \weights_{t+1} \gets \weights_t + \alpha_t \delta_t \nabla_\weights V_\weights(\svec_t) \hspace{0.5cm} \text{ where } \delta_t \defeq C_{t+1} + \gamma_{t+1} V_\weights(\svec_{t+1}) - V_\weights(\svec_{t})
     .
   \end{align*}
With recurrence, where the state is estimated and so is a function of $\weights$, the only difference to this update is in the computation of $\nabla_\weights V_\weights(\svec_t)$, where $\svec_t$ should instead be thought of $\svec_t(\weights)$. This gradient now requires the chain rule, to account for the impact of $\weights$ on the last state, and the state before then, and so on:
   \begin{equation*}
    \frac{\partial V_\weights(\svec_t)}{\partial \weights_i} = \frac{\partial V_\weights(\svec_t)}{\partial \svec_t}^\top \frac{\partial \svec_t}{\partial \weights_i}
   \end{equation*}
where $\svec_t = f_\weights(\svec_{t-1}, \xvec_t)$.
Computing this gradient back-in-time, $\nabla_\weights \svec_t$---which is also called the \emph{sensitivity}---is precisely the aim of most RNN algorithms, including truncated BPTT and RTRL. Any algorithm that computes sensitivities can be used to obtain a TD update with recurrent connections to estimate the state.

For GVFNs, there are two differences: we need to account for off-policy sampling and the fact that state is itself composed of these value estimates, rather than being learned to estimate values. Value estimation within GVFNs requires off-policy updates, because the target policies $\pi_j$ are not typically equal to the behavior policy $\mu$. Therefore, we also need to include importance sampling ratios in the update
\begin{align*}
\rho_{t,j} \defeq \frac{\pi_j(A_t | \hvec_t)}{\mu(A_t | \hvec_t)} \ \ \ \ \text{ for all $j \in \{1, 2, \ldots, \numgvfs\}$}
.
\end{align*}
This ratio multiplies the TD update, to adjust the expectation of the update to be as if action $A_t$ had been taken under $\pi_j$ rather than the behavior $\mu$. For the second difference, the Recurrent TD update is actually even simpler because the value function itself is the state. For the $j$-th value function---which is the $j$-th state variable---we get that $\nabla_\weights \vifunc{j}_\weights$ at time $t$ is $\nabla_\weights \svec_{t,j}$. Notice that this gradient actually corresponds to using the above chain rule update, by using $\vifunc{j}(\svec_t) = \svec_{t,j}$ as a selector function into the state variable.

The \textbf{Recurrent TD} update for GVFNs is
\begin{align}
\svec_t &\gets f_{\weights_t}(\svec_{t-1}, \xvec_t) &&\triangleright \text{ where } \xvec_t \defeq [a_{t-1}, \obs_t] \nonumber\\
\svec_{t+1} &\gets f_{\weights_t}(\svec_{t}, \xvec_{t+1}) &&\triangleright \text{ where } \xvec_{t+1} \defeq [a_{t}, \obs_{t+1}]\nonumber\\
\phivec_{t,j} &\gets \nabla_\weights \svec_{t,j} &&\triangleright \text{ Compute sensitivities using truncated BPTT} \nonumber\\
\delta_{t,j} &\gets C_{t+1}^{(j)} + \gamma_{t+1}^{(j)} \svec_{t+1,j} - \svec_{t,j}   \nonumber\\
\rho_{t,j} &\gets \frac{\pi^{(j)}(a_t | \obs_t)}{\mu(a_t | \obs_t)} &&\triangleright \text{ Policies can be functions of histories, not just of $\obs_t$}  \nonumber\\
\weights_{t+1} &\gets \weights_{t} + \alpha_t \bigg[ \sum_{j=1}^{\numgvfs} \rho_{t,j}\tderror_{t,j} \phivec_{t,j}  \bigg] \label{eq_rtd}
\end{align}

The TD update, however, is only an approximate semi-gradient update, even in the fully observable setting. To obtain exact gradient formulas, we turn to Gradient TD (GTD) algorithms. In particular we extend the nonlinear GTD strategy developed by \citeA{maei2010toward}, to the MSPBNE. As above, we will immediately be able to use any algorithm to compute the sensitivities in the Recurrent GTD algorithm. But, the algorithm becomes more complex, simply because nonlinear GTD is more complex than TD even without recurrence.

We can use the following theorem to facilitate estimating the gradient. The main idea is to introduce an auxiliary weight vector, $\wvec$, to provide a quasi-stationary estimate of part of the objective. This proof and explicit derivation for the resulting Recurrent TD algorithm is given in the appendix. In the main body, we only provide the result for non-compositional GVFs: no GVFs predict the outcomes of other GVFs. This makes the algorithm easier to follow. We prove the more general result and derivation in Appendix \ref{appendix:alg_details_derivs}, in Theorem \ref{thm:gradientsgen}.
\begin{restatable}{theorem}{gradtheorem}\label{thm:gradients}
  Assume that $V_{\weights}(\hvec)$ is twice continuously differentiable as a function of $\weights$ for all
  histories $\hvec\in\Hist$ where $\dw(\hvec)>0$ and that $W(\cdot)$, defined in Equation \eqref{eqn_w}, is non-singular in a small neighbourhood of $\weights$. Assume further that there are no compositional GVFs in the GVFN: no GVFs has a cumulant that corresponds to another GVFs prediction. Then for $W(\weights)$ and $\boldsymbol{\delta}(\weights) $ defined in Lemma \ref{lemma:mspbne_exp},
  \begin{align}
    \wvec(\weights) &\defeq
    W(\weights)^\inv \boldsymbol{\delta}(\weights) \label{eq_secondw}
             \end{align}
  \begin{align}
   % \boldsymbol{\delta}(\weights) &\defeq
     %    \Expected_{d,\mu}\bigg[ \sum_{j=1}^\numgvfs \rho_j(H,A) \tderror_j(H,A,H') \phivec_{j,\weights}(H) \bigg] \nonumber\\
  \hat{\delta}_{j,\theta}(H) &\defeq \phivec_{j,\weights}(H)^\trans \wvec(\weights) \nonumber\\
     \psivec(\weights) &\defeq \Expected_{d, \mu}\left[\sum\limits_{j=1}^{\numgvfs} \rho_j(H,A)\Big(\delta_j(H,A,H') - \hat{\delta}_{j,\theta}(H)\Big)  \nabla^2 \viweights{j}(H)  \wvec(\weights) \right] \label{eq_hv}
  \end{align}
we get the gradient
 \begin{align}
   -\tfrac{1}{2} \nabla  \text{MSPBNE}(\weights) &=
       \boldsymbol{\delta}(\weights) -
       \Expected_{d,\mu}\bigg[\rho_j(H,A)\gamma^{(j)}(H,A,H') \hat{\delta}_{j,\theta}(H) \phivec_{j,\weights}(H') \bigg] - \psivec(\weights) \label{eq_tdc}
 \end{align}
\end{restatable}
We now have two additional terms to estimate beyond the standard sensitivities in a typical RNN gradient. First, we need to estimate this additional weight vector $\wvec$, given in Equation \eqref{eq_secondw}. This can be done using standard techniques in reinforcement learning.
Second, we also need to estimate a Hessian-vector product, given in Equation \eqref{eq_hv}. Fortunately, this can be computed using R-operators, without explicitly computing the Hessian-vector product, using only computation linear in the length of the vector.

The {\bf Recurrent GTD} update, for this simpler setting without composition, is\footnote{As mentioned above, we could have considered an alternative MSPBNE, using an in-development nonlinear MSPBE objective \citep{patterson2020investigating}. The resulting Recurrent GTD algorithm would look very similar, except the Hessian-vector product could be omitted: $\psivec_t$ is simply dropped in the update to $\theta$. }
\begin{align}
\svec_t &\gets f_{\weights_t}(\svec_{t-1}, \xvec_t) \nonumber\\
\svec_{t+1} &\gets f_{\weights_t}(\svec_{t}, \xvec_{t+1}) \nonumber\\
\phivec_{t,j} &\gets \nabla_\weights \svec_{t,j} \hspace{2.0cm} \triangleright \text{ Compute sensitivities using truncated BPTT}  \nonumber\\
\phivec'_{t,j} &\gets \nabla_\weights \svec_{t+1,j}  \nonumber\\
\rho_{t,j} &\gets \frac{\pi^{(j)}(a_t | \obs_t)}{\mu(a_t | \obs_t)}  \nonumber\\
\vvec_t &\gets \nabla^2 \svec_t \wvec_t \hspace{2.0cm} \triangleright \text{ Computed using R-operators, see Appendix \ref{app_gradients}} \nonumber\\
\hat{\delta}_{t,j} &\gets \phivec_{t,j}^\trans  \wvec_t \nonumber\\
  \psivec_t &\gets \sum_{j=1}^{\numgvfs} ( \rho_{t,j}\delta_{t,j} - \hat{\delta}_{t,j})  \vvec_t \nonumber\\
  \weights_{t+1} &\gets \weights_{t} + \alpha_t \bigg[ \sum_{j=1}^{\numgvfs}  \rho_{t,j} \tderror_{t,j} \phivec_{t,j} - \rho_{t,j}  \gamma_{j,t+1} \hat{\delta}_{t,j} \phivec'_{t,j} \bigg] - \alpha_t\psivec_t  \label{eq_rgtd}\\
   \wvec_{t+1} &\gets \wvec_t + \beta_t \bigg[ \sum_{j=1}^{\numgvfs}  \rho_{t,j} (\tderror_{t,j} - \hat{\delta}_{t,j} )\phivec_{t,j} \bigg] \nonumber
\end{align}
The derivation for this algorithm is similar to the derivation for
Gradient TD Networks \citep{silver2012gradient}, though for this more
general setting with GVF Networks. We include additional algorithm
details and derivations in Appendix \ref{appendix:alg_details_derivs},
including the general RGTD algorithm for compositional GVFs in equation \eqref{eq_rgtd_gen}.

As alluded to, there are a variety of possible strategies to optimize the MSPBNE for GVFNs. This variety arises from different strategies to optimize RNNs, back-in-time, as well as from the variety of strategies to optimize the MSPBE for value estimation. For example, we can compute sensitivities using truncated BPTT or RTRL and its many approximations. Similarly, for the MSPBE, there are a variety of different strategies to approximate gradients, because the gradient is not straightforward to sample. These including a variety of gradient TD methods---such as GTD and GTD2---saddlepoint methods and semi-gradient TD (see \citeR{ghiassian2018online} for a more exhaustive list).

\section{Connections to Other Predictive State Approaches}\label{sec_connections}

The idea that an agent's knowledge might be represented as predictions has a long history in machine learning. The first references to such a predictive approach can be found in the work of \citeA{Cunninghambook}, \citeA{becker1973model}, and \citeA{drescher1991made}, who hypothesized that agents would construct their understanding of the world from interaction, rather than human engineering. These ideas inspired work on predictive state representations (PSRs) \citep{littman2001predictive}, as an approach to modeling dynamical systems. Simply put, a PSR can predict all possible interactions between an agent and it's environment by reweighting a minimal collection of core test (sequence of actions and observations) and their predictions, without the need for a finite history or dynamics model.
Extensions to high-dimensional continuous tasks have demonstrated that the predictive approach to dynamical system modeling is competitive with state-of-the-art system identification methods \citep{hsu2012spectral}.
PSRs can be combined with options \citep{wolfe2006predictive}, and some work suggests discovery of the core tests is possible \citep{mccracken2005online}.
One important limitation of the PSR formalism is that the agent's internal representation of state must be composed exclusively of probabilities of action-observation sequences.

A PSR can be represented as a GVF network by using a myopic $\gamma = 0$ and compositional predictions. For a test $q = \action_1\obs_2$, for example, to compute the probability of seeing $\obs_2$ after taking action $\action_1$, the cumulant is $1$ if $\obs_2$ is observed and $0$ otherwise; the policy is to always take action $\action_1$; and the continuation $\gamma = 0$. To get a longer test, say $\action_0\obs_1\action_1\obs_2$, a second GVF can be added which predicts the output of the first GVF. For this second GVF, the cumulant is the prediction from the first GVF (which predicts the probability of seeing $\obs_2$ given $\action_1$ is taken); the policy is to always take action $\action_0$; and the continuation is again $\gamma = 0$. Though GVFNs can represent a PSR, they do not encompass the discovery methods or other nice mathematical properties of PSRs, such as can be obtained with linear PSRs.

TD networks \citep{sutton2004temporal} were introduced after PSRs, and inspired by the PSR approach to state construction that is grounded in observations.
GVFNs build on and are a strict generalization of TD networks.
A TD network \citep{sutton2004temporal} is similarly composed of $\numgvfs$ predictions, and updates using the current observation and previous step predictions like an RNN. TD networks with options \citep{rafols2005using} condition the predictions on temporally extended actions similar to GVF Networks, but do not incorporate several of the recent modernizations around GVFs, including state-dependent discounting and convergent off-policy training methods.
The key differences, then, between GVF Networks and TD networks is in how the question networks are expressed and subsequently how they can be answered.
GVF Networks are less cumbersome to specify, because they use the language of GVFs. Further, once in this language, it is more straightforward to apply algorithms designed for learning GVFs.

More recently, there has been an effort to combine the benefits of PSRs and RNNs. This began with work on Predictive State Inference Machines (PSIMs) \citep{sun2016learning}, for inference in linear dynamical systems. The state is learned in a supervised way, by using statistics of the future $k$ observations as targets for the predictive state. This earlier work focused on inference in linear dynamical systems, and did not state a clear connection to RNNs. Later work more explicitly combines PSRs and RNNs \citep{downey2017predictive,choromanski2018initialization}, but restricts the RNN architecture to a bilinear update to encode the PSR update for predictive state. In parallel, \citeA{venkatraman2017predictive} proposed another strategy to incorporate ideas from PSRs into RNNs, without restricting the RNN architecture, called Predictive State Decoders (PSDs) \citep{venkatraman2017predictive}. Instead of constraining internal state to be predictions about future observations, statistics about future observations are used as auxiliary tasks in the RNN.

Of all these approaches, the most directly related to GVFNs is PSIMs. This connection is most clear from the PSIM objective \citep[Equation 8]{sun2016learning}, where the goal is to make predictive state match a vector of statistics about future outcomes. There are some key differences, mainly due to a focus on offline estimation in PSIMs. The predictive questions in PSIMs are typically about observations 1-step, 2-step up to $k$-steps into the future. To use such targets, batches of data need to be gathered and statistics computed offline to create the targets. Further, the state-update (filtering) function is trained using an alternating minimization strategy, with an algorithm called DAgger, rather than with algorithms for RNNs. Nonetheless, the motivation is similar: using an explicit objective to encourage internal state to be a predictive state.

A natural question, then, is whether the types of questions used by GVFNs provides advantages over PSIMs. Unlike $k$-step predictions in the future, GVFs allow questions about outcomes infinitely far into the far, through the use of cumulative discounted sums. Such predictions, though, do not provide high precision about such future events. As motivated in Section \ref{sec_constraining}, GVFs should be easier to learn online. In our experiments, we include a baseline, called a Forecast Network, that uses $k$-step predictions as predictive features, to provide some evidence that GVFs are more suitable as predictive features for online agents.

\input{experiments_new.tex}

\section{A Discussion on Discovering GVFs for the GVFN}\label{sec_discovery}
%\section{Discovery: Discussion and Future directions}\label{sec_discovery}

In this work, we were constrained to hand-designed GVFs for the GVFNs. While we were able to show several benefits of the framework, this limitation is apparent in Section \ref{section:experiments:poorlyspecified} where we found that poorly specified GVFs---where we intentionally picked GVFs to have magnitude issues or that were difficult to learn---made the GVFN perform poorly as compared to the RNN. This outcome highlights the importance of the next question about how to improve the selection of predictive questions for a GVFN, and how to make this discovery process automatic and situated in the agent's stream of experience. An approach to discovery will also enable GVFNs to be applied to problems in which a set of GVF questions is not immediately apparent, and problems where our simple heuristic methods would create a set too large to manage computationally.

Previous approaches to discovery in predictive representations have focused on finding a set of predictions that would enable the agent to answer all predictive questions accurately. This objective is trying to find a sufficient statistic of the history for all predictions, and has been discussed in various forms \citep{subramanian2020approximate}. This is the approach typically taken in PSRs and a usual criteria when approaching a POMDP problem. This criteria falls naturally from the POMDP specification, where the assumption is there is a true underlying latent state which the agent can determine from enough interactions with the system.  We conjecture that finding such a state is not feasible in large complex problems, and searching for such a state would be a poor use of a finite set of computational resources. Instead, the agent should focus on finding a set of questions which is useful for the agents overarching goals---for example, maximizing the return in the control problem.

In the following section, we describe several prior approaches to discovery applicable to the GVFN framework, develop a simple approach to a discovery framework for future testing, and discuss various ways of specifying GVFs by hand for the GVFN.

\subsection{Previous Approaches}

There are two main families of approaches to discovery of GVFs for GVFNs: generate-and-test and gradient descent.

{\bf Generate and test}
is a natural algorithmic approach when considering a search problem through a complex unordered (or not obviously ordered) space. The core of the approach is to propose GVFs through a generator and approximate their utility for the downstream task through a proxy measure. This approach has been used for representation discovery \citep{mahmood2013representation,javed2020learning}. The simplest setting where such a generate-and-test approach could be used is time series forecasting, as the predictions are on-policy and so policies do not have to be proposed by the generated. Further, practitioners can apply their prior knowledge in creating the cumulant and continuation functions considered by the generator. There are, however, some simple strategies for generating policies, which we discuss in Section \ref{sec_simple_disc}.
%, and later implemented for GVFNs \citep{schlegel2018discover} (the results presented in the workshop paper are included in Section \ref{sec_simple_disc}).

A generate and test algorithm has been developed for TD networks \citep{makino2008line}. The process of discovery involves creating new predictions built entirely from existing structures: senses or predictions. By building new predictions from existing predictions, it facilitates the creation of compositional structures. The system proposed in \citeA{makino2008line} determines when a node (i.e. a prediction or sense) should be expanded on using three criteria. They then expand these nodes in specific ways to ask a broad set of compositional questions. TD networks do not include policies---rather they include action primitives---so the approach does not directly extend. However, the idea of iteratively creating such compositional structures does extend. For example, in this work, the expert network considered in Section \ref{section:experiments:poorlyspecified} was composed of compositional GVFs. Compositional GVFs could be generated simply by using existing GVFs as the cumulant for the new GVF.

{\bf Meta-gradient descent} uses gradient descent to learn meta-parameters that affect learning performance. The meta-parameters could correspond to initialization of a model for later fine tuning \citep{finn2017model}, a set of GVF auxiliary tasks to improve representation learning in Atari \citep{veeriah2019discovery} or parameterized options \citep{bacon2017theoption}. This approach splits the problem into two optimization problems: an inner problem and an outer problem. The inner optimization consists of the usual control or prediction procedure, where the agent seeks to maximize the discounted return or lower prediction error. The outer optimization calculates gradients through this procedure, with respect to the meta-parameters.

For example, to learn a set of GVFs as auxiliary tasks, \citeA{veeriah2019discovery} parameterized the cumulant and continuation functions. They did not need to parameterize the policies for the GVFs as they assumed on-policy prediction: the policy $\pi$ for the GVF is the current policy. These meta-parameters are optimized in the outer loop to produce auxiliary tasks that improve control performance in the inner loop. For our setting, we could similarly parameterize GVF questions, including the policy. This meta-gradient approach was reasonably effective for discovering GVFs as auxiliary tasks, though the procedure is expensive and has some trainability issues. Nonetheless, it is a reasonable direction for pursuing discovery for GVFNs.

\subsection{Investigating a Simple Generate and Test Strategy for GVF Discovery} \label{sec_simple_disc}

\begin{wrapfigure}{R}{0.5\textwidth}
  \centering
  \includegraphics[width=0.4\textwidth]{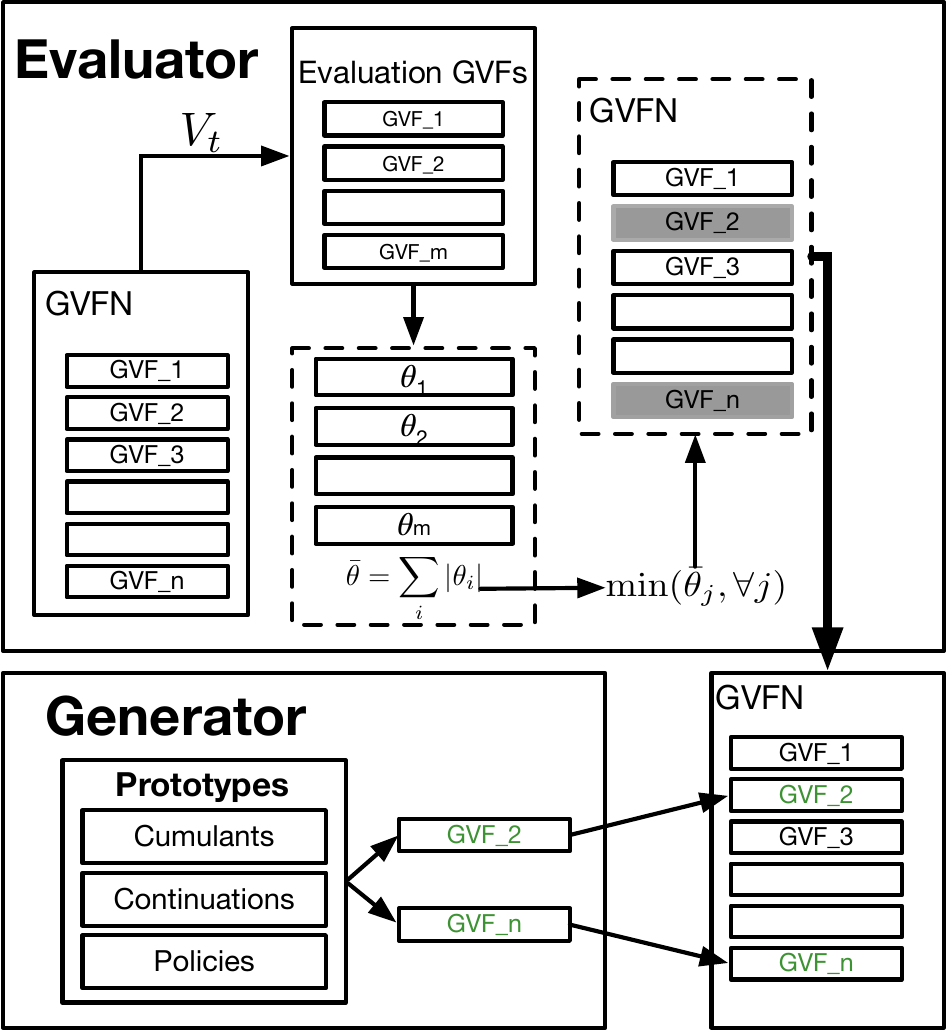}
  \caption{The discovery framework.}\label{fig_discovery}
\end{wrapfigure}
We base this simple discovery framework on algorithms described for representation
search \citep{mahmood2013representation} focusing on two main components: an
evaluator, and a generator. The evaluator is responsible for testing GVFs and
removing unused GVFs. The generator proposes new GVFs from a set of possible
GVFs. We summarize our framework in Figure \ref{fig_discovery}. The key questions are how GVFs are evaluated and how new ones proposed. Our goal here is simply to demonstrate one avenue for discovery in GVFNs, rather than to develop an algorithm for discovery; we therefore opt for what we believe are some of the simplest choices.

To evaluate the usefulness of a GVF we look at the magnitude of the associated weight in the external tasks using the GVFN. We assume the state vector is used linearly to make predictions, with $\theta_j$ corresponding to state $s_j$ and so to the $j$th GVF. We evaluate all the GVFs every $N \in \Naturals$ steps and prune the lowest $\epsilon \in [0,1]$ percentage, i.e., prune $\lfloor n \epsilon \rfloor$ least useful GVFs of the full set of $\numgvfs$ GVFs. Other criteria have been proposed for evaluation, such as using traces of the weight magnitudes and considering internal weights \citep{mahmood2013representation}. As mentioned above, we opt for the simplest choice that is still reasonably effective.

We generate new GVFs randomly from a set of GVF primitives. We define a set of basic types of cumulants, continuations and policies from which to randomly sample. For continuations, we consider
\textit{myopic discounts} ($\gamma = 0$),
\textit{horizon discounts} ($\gamma \in (0,1)$) and
\textit{terminating discounts} (the discount is set to $\gamma \in (0,1]$ everywhere, except for at an
event, which consists of a transition $(o, a, o')$).
For cumulants, we consider
\textit{stimuli cumulants} (the cumulant is one of the observations,
or taking on 0 or 1 if the observation fulfills some criteria (e.g. a threshold))
and \textit{compositional cumulants} (the cumulant is the prediction of another GVF).
We also use \textit{random cumulants} (the cumulant is a random number generated from a zero-mean Gaussian with a random variance sampled from a uniform distribution); we do not expect these to be useful, but rather use it to define what we call a dysfunctional GVF to test pruning.
For the policies, we propose \textit{random policies} (an action is chosen at random) and
\textit{persistent policies} (always follows one action).

The resulting GVF primitives consist of a triplet $(c, \gamma, \pi)$ where each is randomly chosen from these basic types. For example, a randomly generated GVF could consist of a myopic continuation, a stimuli cumulant on observation bit one and a random policy. This would correspond to predicting the first component of the observation vector on the next step, assuming a random action is taken. As another example, a randomly generated GVF could consist of a termination continuation with $\gamma = 0.9$, a stimuli cumulant which is 1 when the observation is zero and is otherwise zero otherwise and a persistent policy with action forward. This GVF corresponds to predicting the likelihood of seeing the observation change from active (`1') to inactive (`0'), given the agent persistently moves forward, within the horizon of about $(1-\gamma)^\inv = 10$ steps.

We could also have considered parameterized continuations, cumulants and policies and randomly sample from that set. This set, however, is large. The GVF primitives can be seen as a prior over the full set of GVFs, which is too large from which to randomly generate. Without this prior we expect the discovery approach to still work but to take even longer than the experiments we present here.

We evaluate the performance of our system on two experiments in Compass World \citep{sutton2005temporal}. Both experiments use the five hard-to-learn GVFs as the targets for the GVFN, introduced in Section \ref{section:experiments:compassworld}. These questions correspond to a question of ``which wall will I hit if I move forward forever?''.
The first experiment, Figure \ref{fig:compass_disc} (left), provides a sanity check that the evaluation strategy prunes dysfunctional representational units. We initialize the GVF network with 200 GVFs: 45 used to form the expert crafted TD network \citep{sutton2005temporal}, and 155 defective GVFs predicting noise $\sim \mathcal{N}(0,\sigma^2)$. We report the learning curve and pruned GVFs over 12 million steps. The second experiment, Figure \ref{fig:compass_disc} (right), uses the full discovery approach to find a representation useful for learning the evaluation GVFs. We report the learning curves of the evaluative GVFs over 100 million steps.

These experiments have many similarities to the experiments above, but there is one key differences worth noting. Instead of using RTD or RGTD, we used TD($\lambda$); see Appendix \ref{app_tdlambda} for the update equations. We found that this was sufficient to learn the expert network specification in a reasonable number of steps, and is significantly simpler than the other algorithms. Note that we did not use this algorithm in the above comparisons with RNNs, for two reasons. First in the cases where the target was not a return, it is not possible to use eligibility traces, as they are designed for predicting expected returns. Second, as far as we are aware, the eligibility trace calculation for neural networks with several output nodes has not been formally derived nor tested.
%Another difference is in the included baselines. We include two. The first is performance of learning without any recursive state (from raw observations), and the second is a strategy that does random generation once ($\epsilon = 0.0$).

\begin{figure}[t]
  \centering
  \includegraphics[width=\textwidth]{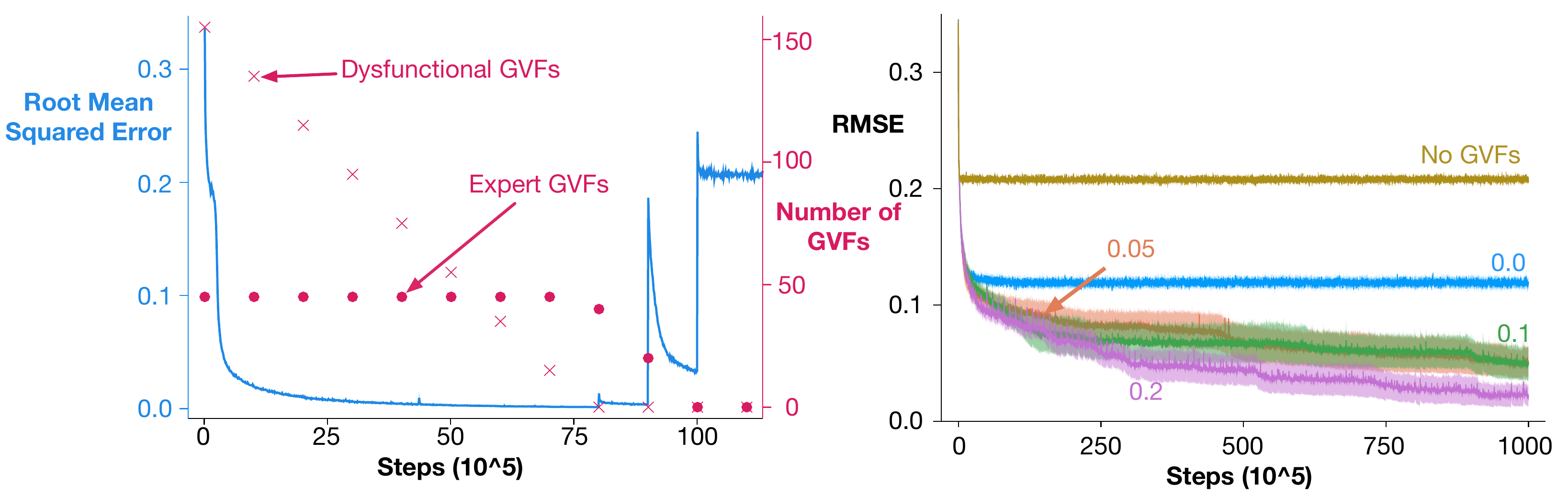}
  \caption{\textbf{(left)} Pruning predictive units occurs every million steps with no regeneration $\alpha=0.001, \lambda=0.9, \epsilon=0.1, \sigma^2 = 1$ \textbf{(right)} Learning curves of the evaluative GVFs $N = 1000000, \epsilon=\text{labeled}, \alpha=0.001, \lambda=0.9, n = 100$, over 5 runs with standard error denoted by the shaded region.
    }\label{fig:compass_disc}
\end{figure}

The results indicate that even a simple generate and test approach can be effective for discovery in GVFNs. The first figure shows that the pruning approach gradually removes the dysfunctional GVFs, without pruning the expert GVFs. Eventually, once the agent has mostly removed all the dysfunctional GVFs, it is then forced to prune the expert GVFs and prediction performance begins to drop. Of course, in practice, the agent would not prune all its GVFs; in this experiment we simply continue the pruning until the end to avoid biasing when we stop the agent.
The second plot shows that iteratively pruning and generating new GVFs significantly improves on using an initial random set. For $\epsilon = 0.2$, which means about 20\% of GVFs are pruned in each pruning phase, the prediction error continues to decrease until it almost reaches 0 and is almost as good as the set of hand-design GVFs used in previous experiments.

The goal of this experiment was to answer: is it possible to discover useful GVFs for a GVFN, even in simple settings? A negative answer would mean that GVFNs might have limited applicability. A  demonstration that it is possible provides some evidence that this is a tractable problem for which even simple solutions can help us make traction. This demonstration, however, by no means shows an ideal or even efficient algorithm and there is ample room for improvement. Primarily, the random generation strategy does not take into account the current set of proposed predictions, potentially resulting in redundancy. A more principled method would look to generate a wide variety of predictions dependent on the current set of predictions.
%This would involve measuring how related GVF questions are from their specification is not particularly straightforward.
Another issue is the proxy used to determine a prediction's usefulness. Currently, the system will prune GVFs that are not directly useful, even if they are the cumulant for a useful GVF. The cumulant for the useful GVF is replaced by a new random GVF. This could reduce the quality of the predictive state or cause other instabilities within the GVFN. A simple approach is to define usefulness based also on compositional utility, not just on utility for the prediction task. The usefulness of a GVF should be higher if it is used by a GVF that is itself heavily relied on for accurate predictions, versus if it is only used by less useful GVFs.

\subsection{Heuristics to specify GVFNs}

Through testing GVFNs in several domains we have developed some rules of thumb for choosing GVFs which can be used today. In our time-series experiments, we found selecting GVFs with constant $\gammaj{j} \in [1 - 2^{-j}]$ to be surprisingly effective across the settings with fixed policies---namely the time series datasets. This is encouraging as these specifications on the surface seem simpler to discover than something as complex as the Expert network in Compass World. A set of discounts selected linearly across a range was also effective. We also found that including GVFs which have a pseudo-termination at a known event (known due to expert knowledge) and a cumulant which is only active at this event improved learning performance considerably (see the performance of the Terminating-Horizon network in Section \ref{section:experiments:poorlyspecified}).

\section{Discussion and Conclusions}

In this work, we made a case for a new recurrent architecture, called GVF Networks.
GVFNs constrain the hidden layer to correspond to predictions about the future, and so can be seen as a regularized or constrained RNN architecture. We first derive a sound fixed-point objective for these networks. We then show in experiments that GVFNs can outperform various RNN architectures with a much smaller truncation in BPTT. We demonstrated this phenomena on time series data as well as a RL prediction environment designed to have longer term dependencies. The goal of the paper was to further investigate the predictive representation hypothesis, where we asked if it is useful for trainability to restrict hidden states to be predictions. The work provided simpler algorithms than previous related work, such as TD networks, to test this hypothesis, as well as some evidence that restricting hidden state to be prediction can be beneficial. We finally investigated the impact of the specification of these predictions, and demonstrated that careful curation---an expert set---of GVFs could improve performance, but that relatively simple heuristics were also quite effective. We also found, though, that poorly specified GVFs---where we intentionally picked GVFs to have magnitude issues or that were difficult to learn---made the GVFN perform poorly as compared to the RNN.

% This outcome leads to an important next question about how to improve how we pick predictions for GVFNs.
% In this work, we assumed that the GVFs for the GVFN were given. A critical question to make GVFNs easy-to-use is to answer the discovery question: how to autonomously discover these questions. The discovery problem is an important next step, and beyond the scope of this work. We can, however, provide some rules of thumb for choosing GVFs today, without yet having an answer to the discovery question. In our time-series experiments, we found selecting GVFs with constant $\gammaj{j} \in [1 - 2^{-j}]$ to be surprisingly effective across the settings with fixed policies---namely the time series datasets. This is encouraging as these specifications on the surface seem simpler to discover than something as complex as the Expert network in Compass World. We also see good (albeit slow) results with a random generation strategy, generating GVFs from a set of simpler primitives (see Appendix \ref{appendix:discovery} for details and results). Another approach could be to use meta-gradient descent to learn the set of GVFs \cite{veeriah2019discovery}, and should be explored in future work.

In addition to trainability, constraining features to be predictions has other potential benefits we did not directly demonstrate in this work, primarily for transfer and adapting to changes in the environment. Predictive features can be useful for transfer because they can encode knowledge about dynamics that remain consistent, even when the agent has to make a new prediction or find a policy for a different reward function. Further, forgetting is a known problem with neural networks when transferring to new problems \citep{mccloskey1989catastrophic,kemker2018measuring,maltoni2019continuous}; by primarily using the GVF prediction loss for the state update, it could alleviate some of these forgetting issues.
% MARTHAC: I think this is already said above now
%As mentioned below the predictive state may need to be complemented with the state learned from a traditional recurrent architecture to improve final performance, but the already learned predictive state could make initial learning faster.
% MARTHAC: I think here you were saying you could keep an old network, based on the end of an episode?
%This would especially benefit the lifelong learning setting where the lack of termination make re-initializing the agent's learning progress a difficult decision.
The predictions learned in the state could also provide important information about the agent's experience, such as features that predict surprise \citep{gunther2018predictions}, knowledge about the variability of the environment \citep{sherstan2018directly}, and other various statistics \citep{modayil2014multi,white2015thesis,sherstan2020representation}.

A related point is that predictive features could facilitate adapting more quickly when (part of) the world changes. There are two reasons for this. First, even if the targets for a GVF change, the GVF itself might still be a useful prediction to use in the state. By directly updating state to correspond to predictions, the features update quickly to respond to the change. This is in contrast to an RNN, where the feature update is more indirect. Second, even if parts of the world change, many predictions may remain accurate and pertinent. These features will remain stable even under the change, and only a smaller number of features, that need to change, will change. This property allows us to better re-use previous learning and promote stability.

% MARTHAC: I felt this had too many ideas.
%When considering the amount of experience possible required by a discovery algorithm, a natural question you might have is why we would want to use a portion of an agent's experience to find these predictive questions when a recurrent network learns a state directly from the agent's objective. While we showed the potential for GVFNs to cut down on the need for calculating gradients through a history, we conjecture a predictive state would be better at transfer and more stable than its RNN counterpart. This conjecture comes from the observation that by decoupling the state (or at least part of the state) from the learning process of the overall task, we can more readily apply the dynamics we've learned to new problems. This is because no matter the final problem setting, the predictions should remain consistent and not need time to re-learn. The state learned through a recurrent network will need time to adjust to the new problem setting and will likely run into problems when transferring to the new problem.

In addition to testing these potential benefits as a next step, there are a few algorithmic extensions that are promising as well.
The architecture proposed here is only tested in the online prediction setting, but we expect to see similar benefits in other settings in which RNNs are employed. One of particular interest is in applying GVFNs to the control setting. This can be easily done through having the final network output be a state-action value function as in a Deep Q-Network \citep{mnih2015human,hausknecht2015deep}, or by using the state of the GVFN as input to an actor-critic algorithm such as Impala \citep{espeholt2018impala}. Predictive representations built using GVFs have been shown to be advantageous in real-world control applications, such as autonomous driving \citep{graves2020learning}, and we expect GVFNs to share similar properties in settings where state construction is necessary.

Finally, in this work, we constrained ourselves to GVFNs where all hidden states are predictions. A natural extension is to consider a GVFN that only constrains certain hidden
states to be predictions and otherwise allows other states to simply be set to
improve prediction accuracy for the targets. This modification could provide the improved stability of GVFNs, but improve representability. Additionally,
GVFNs could even be combined with other RNN types, like LSTMs, by simply
concatenating the states learned by the two RNN types. Overall, GVFNs provide a complementary addition to
the many other RNN architectures available, particularly for continual learning systems with long temporal dependencies; with this work, we hope to expand interest and investigation further into these promising architectures.

\section{Acknowledgements}
We would like to thank the Alberta Machine Intelligence Institute, IVADO, NSERC and the Canada CIFAR AI Chairs Program for the funding for this research, as well as Compute Canada for the computing resources used for this work. We would also like to thank Marc Bellemare for helpful comments about extensions to infinite sets for the set of histories.

\appendix

\input{appendix}

{
%\footnotesize
%\bibliographystyle{unsrtnat} %

%\bibliographystyle{abbrvnat}
  \bibliography{paper}
  \bibliographystyle{theapa}
}

\end{document}

%% file: experiments_new.tex
% !TEX root = paper.tex

\section{Experiments in Forecasting} \label{sec:exp_forecasting}

In this section, we compare GVFNs and RNNs on two time series prediction datasets, particularly to ask 1) can GVFNs obtain comparable performance and 2) do GVFNs allow for faster learning, due to the regularizing effect of constraining the state to be predictions.\footnote{All code for these experiments can be found at \url{https://github.com/mkschleg/GVFN}} We investigate if they allow for faster learning both by examining learning speed as well as robustness to truncation length in BPTT.

\paragraph{Datasets}

We consider two time series datasets previously studied in a comparative
analysis of RNN architectures by \cite{bianchi2017anoverview}: the Mackey-Glass
time series (previously introduced), and the Multiple Superimposed Oscillator.

The single-variate {\bf Mackey-Glass (MG)} time series dataset is a synthetic data set generated from a time-delay differential equation:
\begin{equation}
  \partialderivative{y(t)}{t} = \alpha\frac{y(t-\tau)}{1+y(t-\tau)^{10}} - \beta y(t)\label{eq:mg}
  .
\end{equation}
We follow the learning setup in \cite{bianchi2017anoverview}: we set $\tau=17$,
$\alpha=0.2$, $\beta=0.1$, and we take integration steps of size $0.1$. We forecast the target variable $y$ twelve steps into the
future, starting from an initial value $y(0)=1.2$. We generate $\nsamples = 600,000$ samples.

The {\bf Multiple Superimposed Oscillator (MSO)} synthetic time series \citep{jaeger2004harnessing} is defined by the sum of four sinusoids with unique frequencies

\begin{equation}
  y(t) = \sin(0.2t)+\sin(0.311t)+\sin(0.42t)+\sin(0.51t). \label{eq:mso}
\end{equation}
The resulting oscillator has a long period of $2000\pi \approx 6283.19$. Because we generate data using $t\in\Naturals$, the oscillator effectively never returns to a previously seen state. These attributes make prediction difficult with the MSO, as the model cannot rely on memory alone to make good predictions. We generate $\nsamples = 600,000$ samples and make predictions with a forecast horizon of $h=12$.

\paragraph{Experiment Settings}

The focus in this work is on online prediction, and so we report online prediction error. At each step $t$, after observing $o_t = y(t)$, the RNN (or GVFN) makes a prediction $\hat{y}_t$ about the target $y_t$, which is the observation 12 steps into the future, $y_t = y(t+h)$. The magnitude of the squared error $(\hat{y}_t - y_t)^2$ depends on the scale of $y_t$. To provide a more scale invariant error, we normalize by the mean of the target---a mean predictor. Specifically, for each run, we report average error over windows of size 10000 with the mean predictor is computed for each window. This results in $\nsamples/10000$ normalized squared errors, where $\nsamples$ is the length of the time series. We repeat this process 30 times, and average these errors across the 30 runs, and take the square root, to get a Normalized Root Mean Squared Error (NRMSE).

We fixed the values for hyperparameters as much as possible, using the previously reported value for the RNN and reasonable defaults for the GVFN. The stepsize is typically difficult to pick ahead of time, and so we sweep that hyperparameter for all the algorithms.  We attempted to make the number of hyperparameters swept comparable for all methods, to avoid an unfair advantage. We do not tune the truncation length, as we report results for each truncation length $p\in \{1, 2, 4, 8, 16, 32\}$ for all the algorithms.

\paragraph{Algorithm Details}

The GVFN consists of a single layer of size 32 and 128 (for MG and MSO respectively), corresponding to horizon GVFs. As described in Section \ref{sec:case_study}, each GVF has a constant continuation $\gamma^{(j)} \in [0.2,0.95]$ and cumulant $C_{t}^{(j)}=\frac{1-\gamma^{(j)}}{y^{\text{max}}_{t}}y(t)$, where
$y^{\text{max}}_t$ is an incrementally-computed maximum of the observations
$y(t)$ up to time $t$. The GVFs are generated to linearly cover the range
$[0.2,0.95]$. This set is chosen as one of the simplest options that can be used
without much domain knowledge. It is likely not the optimal set of GVFs for the
GVFN, but represents a reasonable default choice.
The GVFN is followed by
a fully-connected layer with relu activations to produce a non-linear
representation, which is linearly weighted to predict the target.
The GVFN layer uses a linear activation, with clipping between [-10,
10], to help ensure state features remain bounded;
again, this represented a simple rather than optimized choice.

The GVFN was
trained using Recurrent TD with a constant learning rate and a batch size of 32. The weights for the fully-connected relu layer and the weights for
the linear output  are trained using ADAM, to minimize the mean squared error
between the prediction at time $t$ and target $y(t+h)$. We swept the stepsize
hyperparameters: the learning rate for the
GVFN $\alpha_{\text{\tiny GVFN}} = N\cdot10^{-k}$ for $N\in\{1,5\}$, $k\in\{3,\ldots,6\}$,
and the learning rate for the fully-connected and output layers $\alpha_{\text{pred}} =N\cdot10^{-k}$ for
$N\in\{1,5\}$, $k\in\{2,\ldots,5\}$.

We compare to RNNs, LSTMs, and GRUs \footnote{We use standard implementations found in Flux \citep{innes2018}.}. The network architecture is similar to the GVFN for all recurrent models. The RNN size is set to 32 for MG and 128 for MSO, while the GRU and LSTM have 8 hidden units for MG and 128 for MSO. Notice how the GRU and LSTM have fewer hidden units than the RNN and GVFN for the MG experiment. This roughly accounts for the increased complexity of the LSTMs and GRUs as compared to the GVFN and RNN. While this was needed to make all the models competitive in MG, we found the GVFNs performed well in MSO even with the same number of hidden units as the GRU and LSTMs.
We trained these models using p-BPTT---specifically with the ADAM optimizer with a batch size of 32---to
minimize the mean squared error between the prediction at time $t$ and $y(t+h)$. We swept the learning rate $\alpha = 2^{-k}$ with $k \in \{1,\ldots,20\}$.
%
%, and truncation length $p = 2^k$ with $k \in [0, \ldots, 6]$.
%
%
\begin{figure}[t!]
  \center
  \includegraphics[width=0.95\textwidth]{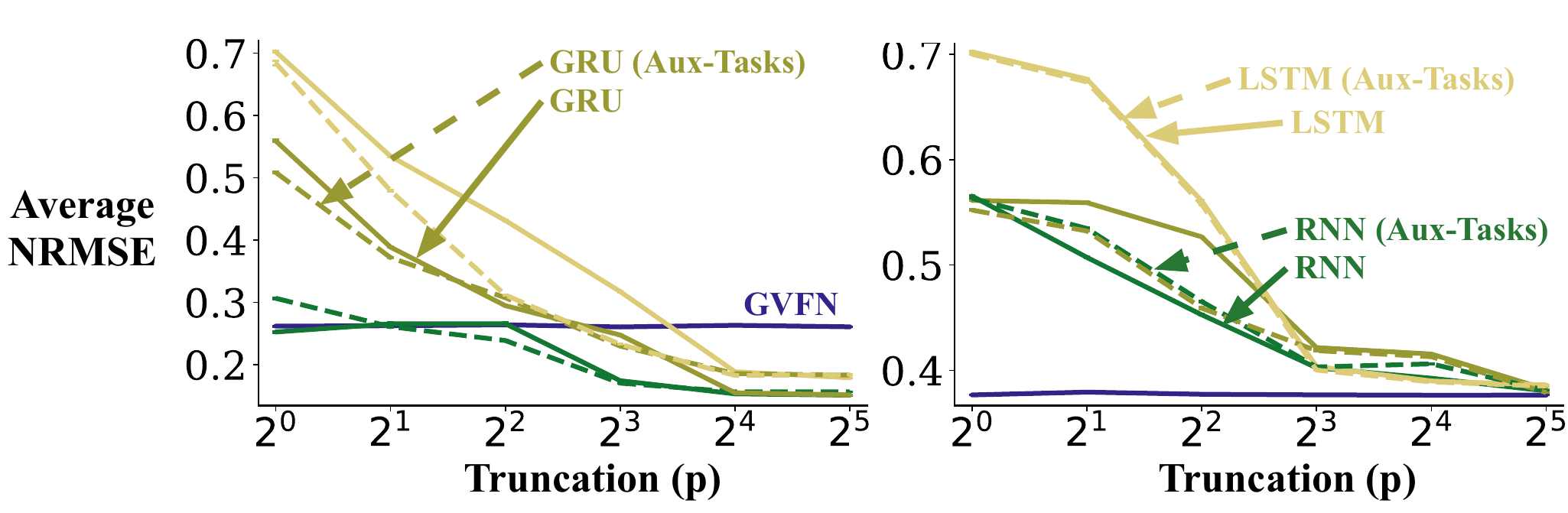}
  % \begin{subfigure}{0.4\textwidth}
  %   \includegraphics[width=\textwidth]{plot_jair/timeseries/raw/mg_tau_sens_train.pdf}
  % \end{subfigure}
  % \begin{subfigure}{0.4\textwidth}
  %   \includegraphics[width=\textwidth]{plot_jair/timeseries/raw/mso_tau_sens_train.pdf}
  % \end{subfigure}
  \caption{
    % Learning curves for the (\textbf{top}) Mackey-Glass and (\textbf{bottom}) Multiple Superimposed Oscillator datasets. We are reporting the normalized root mean squared error (NRMSE) normalized to the performance of the windowed average baseline. We use the average of 30 independent runs $\pm$ the standard error.
    Truncation sensitivity for the (\textbf{left}) Mackey-Glass and (\textbf{right}) Multiple Superimposed Oscillator datasets. Errors are calculated using the normalized root mean squared error (NRMSE) averaged over the last 10k steps for the training results $\pm$ 1 standard error over 30 independent runs.
    % (\textbf{right}) ACEA electrical load ($\tau_{ACEA} = 10$).
  }\label{fig:timeseries_sens}
\end{figure}

\begin{figure}[t!]
  \centering
  % \includegraphics[width=\textwidth]{plots/timeseries.pdf}
  % \includegraphics[width=\textwidth]{plot_jair/timeseries/trunc_comb.pdf}
  % \begin{subfigure}{\textwidth}
  %   \includegraphics[width=\textwidth]{plot_jair/timeseries/raw/mg_lc.pdf}
  %   \vspace{0.1mm}
  % \end{subfigure}
  % \begin{subfigure}{\textwidth}
  %   \includegraphics[width=\textwidth]{plot_jair/timeseries/raw/mso_lc.pdf}
  % \end{subfigure}
  \includegraphics[width=0.95\textwidth]{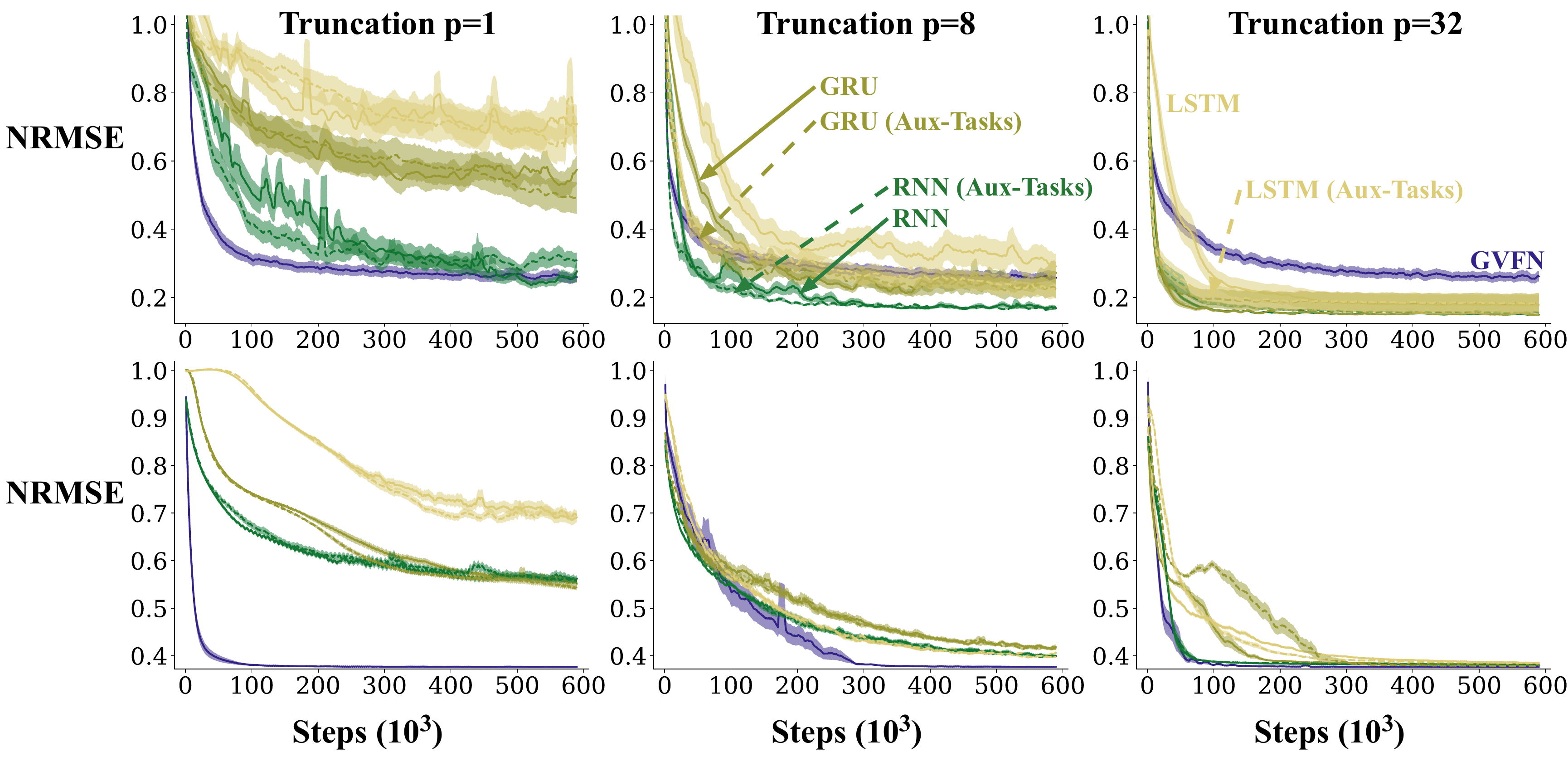}
  \caption{
    Learning curves for the (\textbf{top}) Mackey-Glass and (\textbf{bottom}) Multiple Superimposed Oscillator datasets. We are reporting the normalized root mean squared error (NRMSE) normalized to the performance of the windowed average baseline. We use the average of 30 independent runs $\pm$ the standard error.
    % Truncation sensitivity for the three time-series datasets for training and testing data sets. Errors are calculated using the normalized root mean squared error (NRMSE) averaged over the last 10k steps for the training results and the entire testing set $\pm$ 1 standard error over 10 independent runs. (\textbf{top}) Mackey-Glass ($\tau_{MG} = 100$) (restated from above) (\textbf{bottom}) Multiple Superimposed Oscillator ($\tau_{MSO} = 100$).
    % (\textbf{right}) ACEA electrical load ($\tau_{ACEA} = 10$).
  }\label{fig:timeseries_lc}
\end{figure}

Finally, we also compare to RNNs with the 128 GVFs as auxiliary tasks.
The augmented RNN has the same architecture as above, but with an additional set of output heads. The additional GVF heads are the same as those used by the GVFN, and
are trained with TD. The gradient information from the GVFs is back-propagated through the network, influencing the representation. The augmented RNN was tuned
over the same values as the RNN. The goal for adding this baseline is to gauge if there is an important difference in using the GVFs to directly constrain the state, as opposed to indirectly as auxiliary tasks. It further ensures that the RNN is given the same prior knowledge as the GVFN---namely the pertinence of these predictions---to avoid the inclusion of prior knowledge as a confounding factor.

All RNNs and GVFNs include a bias unit, as part of the input as well as in all layers. All methods have similar computation per step, particularly as they are run with the same truncation levels $p$.

\paragraph{Results}
We first show overall results across the truncation level in p-BPTT in Figure
\ref{fig:timeseries_sens}. Three results are consistent across both datatsets: 1) GVFNs can obtain
significantly better performance than RNNs with small $p$; 2) GVFNs are surprisingly robust to truncation level, providing almost the same performance across $p$; and 3) auxiliary tasks in the RNN do not provide consistent benefits across models and datasets. GVFNs provide a strict improvement on the MSO dataset.
The result on MG is more nuanced. As truncation levels increase, the RNN's performance significantly improves and then passes the GVFN. This might suggest some bias in the specification of the GVFs. As is typical with regularization or imposing an inductive bias, it can improve learning---here allowing for much more stable learning with small $p$---but can prevent the solution from reaching the same prediction accuracy. In some cases, if we are fortunate, the inductive bias is strictly helpful, constraining the solution in the right way so as to incur minimal bias but improve learning. In MSO, it's possible the GVF specification was more appropriate and in MG less appropriate.

%This is further supported by the fact that the auxiliary tasks were more harmful in MG than MSO.

To gain more detailed insight into the behavior of the algorithms across truncation levels, we show learning curves for $p \in \{1, 8, 32\}$ in Figure \ref{fig:timeseries_lc}. All the approaches learn more slowly for $p =1$, but the RNNs are clearly impacted more significantly. In MSO, the GVFN has a clear advantage in terms of learning speed. This is not true in MG, where once $p \ge 8$, the RNN performs better and learns faster. The GVFN objective here may actually be difficult to optimize, but it allows the agent to make progress constructing a useful state, whereas the signal from the error to the targets is insufficient.

\section{Investigating Performance under Longer Temporal Dependencies}\label{section:experiments:compassworld}

\begin{figure}[t]
  \center
  \includegraphics[width=0.95\textwidth]{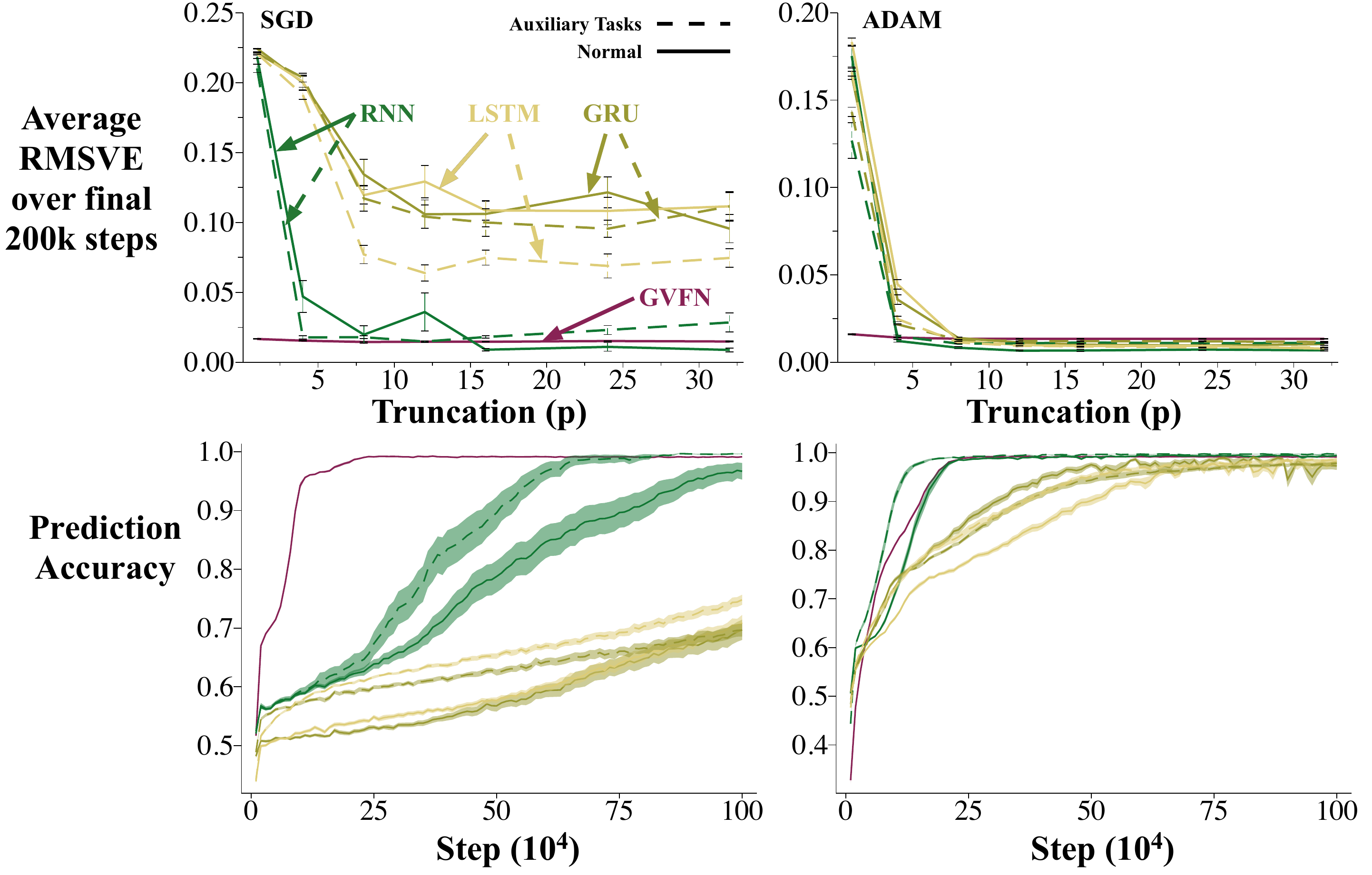}
  \caption{Results averaged over 30 runs $\pm$ one standard error. The dashed lines correspond to each RNN type augmented with auxiliary tasks, namely here the terminating horizon GVFs. The plots on the \textbf{(left)} are for a constant learning rate swept in range $\{0.1\times1.5^i; i \in [-10, 5]\} \cup \{1.0\}$. The plots on the \textbf{(right)} are for the ADAM optimizer with learning rate swept in range $\{0.01\times 1.5^i; i \in \{-18, -16, \ldots, 0\}\}$. The \textbf{(top)} row shows sensitivity over truncation measured by the average root mean squared value error (RMSVE) over the final 200000 steps of training. The \textbf{(bottom)} row shows learning curves for $p=4$ for prediction accuracy. We check if the prediction is correct by predicting the color of five with the highest GVF output, where the GVF prediction corresponds to a probability of facing that wall. When averaged over a window (10000 steps in our case) this results in a percentage of correct predictions during that time span.
  } \label{fig:compass}
\end{figure}

In this section, we investigate the utility of constraining states to be predictions, for an environment with long temporal dependencies. We use Compass World, introduced in Section \ref{GVFNs} (see Figure \ref{fig:compass_world_env}), which can have long temporal dependencies, because the random behavior can stay in the center of the world for many steps, observing only the color white.
The observation is encoded with two bits per color: one to indicate the agent observes that color, and the other to indicate another color is observed. The behavior policy chooses randomly between moving one-step forward; turning right/left for one step; moving forward until the wall is reached ({\em{leap}}); or randomly selecting actions for $k$ steps ({\em{wander}}). The full observation vector is encoded based on which action was taken, and includes a bias unit.

We chose five hard-to-learn GVFs with predictions corresponding to the wall the agent is facing. These predictions are not learnable without constructing an internal state. These five questions correspond to leap questions. The leap question is defined as having a cumulant of 1 in the event of seeing a specific wall (orange, yellow, red, blue, green), and a continuation function defined as $\gamma = 0$ when any color is observed---when the agent is facing a wall---and $\gamma = 1$ otherwise.

We use the same architecture for both RNNs and GVFNs; the main difference is that for the GVFN we constrain the hidden state to be GVF predictions. The GVFN uses 40 GVFs: 8 GVFs per color. The 8 GVFs for a color correspond to \textbf{Terminating Horizon} GVFs. This means that they have a cumulant of 1 when seeing that color, and zero otherwise; they have a $\gamma = 1- 2^k$ for one of 8 $k \in \{-7, -6, \ldots, -1\}$; they terminate---$\gamma$ becomes zero---when any color is observed; and the policy is to always go forward. These GVFs are similar to the horizon GVFs in time series prediction, except that termination occurs when a wall is reached and the policy is off-policy.
The RNN similarly uses 40 hidden units for the recurrent layer. For RNNs, we use the hyperbolic tangent and the sigmoid function for GVFNs. We used sigmoids instead for GVFNs, because the returns are always nonnegative; otherwise, these two activations represent a similar architectural choice.

We found treating the input action $a_t$ specially significantly improved performance of both the RNN and GVFN. This is done by specifying separate weight vectors $\{w_a \in \mathbb{R}^n ; \forall a \in \mathcal{A}\}$ for each action the agent can take. The hidden state is then calculated as $\svec_{t+1} = \sigma(w_{a_t}^\trans[\xvec_{t+1}, \svec_t])$, where $\sigma$ is the activation function. For the GRUs and LSTMs, this architectural modification is not straightforward; instead we pass the action as a one-hot encoding.

All the approaches share the same structure following the recurrent layer. The state $\svec_t$ is passed to a 32-dimensional hidden layer with relu activation, and then is linearly weighted to produce the predictions for the five hard-to-learn GVFs: $\hat{\mathbf{y}}_t = \text{relu}(\svec_t^\top \mathbf{F}) \mathbf{W}$ where $\mathbf{F} \in \RR^{40 \times 32}$ and $\mathbf{W} \in \RR^{32 \times 5}$. All methods include a bias unit on every layer.

The performance for increasing $p$, as well as learning curves for $p = 8$, are show in Figure \ref{fig:compass}. Again, we obtain a several clear conclusions. 1) The GVFN is again highly robust to truncation level, reaching almost perfect accuracy with $p =1$. 2) The GVFN can learn noticeably faster with smaller $p$, such as $p = 4$, and the differences disappear for larger $p$. 3) The auxiliary tasks do not provide near the same level of benefit as the GVFN, though unlike the time series results, there does in fact seem to be some benefit. 4) All the methods are improved when using ADAM---especially the LSTMs and GRUs---though GVFNs are effective even with constant stepsizes.

\section{Investigating Poorly Specified GVFNs}\label{section:experiments:poorlyspecified}

In the previous Compass World and Forecasting experiments, the GVFNs were robust to truncation. In fact, computing one-step gradients was sufficient for good performance. A natural question is when we can expect this to fail. We hypothesize that this robustness to truncation relies on appropriately specifying the GVFs in the GVFN. Poorly specified GVFs could both (a) make it so that the GVFN is incapable of constructing a state that can accurately predict the target and (b) make training difficult or unstable. In this section, we test this hypothesis by testing several choices for the GVFs in the GVFN in Compass World.

We consider three additional GVFN specifications: two that include intentional (but realistic) misspecifications and one that should be an improvement on the Terminating Horizon GVFN. The first misspecification, which we call the \textbf{Horizon} GVFN, causes the hidden states to have widely varying magnitudes. These GVFs are similar to the Terminating Horizon GVFs, except that they do not include termination when a color is observed. This means the true expected returns can be quite large, up to $\tfrac{1}{1-\gamma}$ (e.g.,  $\frac{1}{1-0.99} = 100$) if the agent is already immediately in front of the wall with that color. The policy is to go forward, and so if the agent is already facing the wall and receives a cumulant of 1, it will see a 1 forever onward, resulting in a return of $\sum_{i=0}^\infty \gamma^i = \tfrac{1}{1-\gamma}$.

The second misspecification provides a minimal set of sufficient predictions, but ones that are harder to learn. A natural choice for this is to use the five hard-to-learn predictions themselves, which is clearly sufficient but may be ineffective because we cannot learn them quickly enough to be a useful state. We call this the \textbf{Naive} GVFN, because it naively assumes that representability is enough, without considering learnability.

Finally, we also consider a specification that could improve on the more generic Terminating Horizon GVFN, that we call the \textbf{Expert Network}. This network also has 40 GVFs, but ones that are hand-designed for Compass World. This GVFN is a modified version of the TD network designed for Compass World \citep{sutton2005temporal}. The GVFs are defined similarly for the 5 colours. There are 3 myopic GVFs: a myopic GVFs consists of a myopic termination ($\gamma = 0$ always) and a cumulant of the color bit. Each myopic GVF has a persistent policy, which takes one action forever. Since there are three actions there are three myopic GVFs. These myopic GVFs indicate whether the agent is right beside the color (ahead, to the left or to the right). There is 1 leap GVF where the policy goes forward always, the cumulant is again the color bit and $\gamma=1$ except when a color is observed, giving $\gamma =0$. There are 2 GVFs with a persistent policy (left, right) with myopic termination and a cumulant of the previous leap GVF's. These compositional GVFs let the agent know if they were to first turn right (or left) and then go forward, would they see the color. There are 2 leap GVFs with cumulants of the myopic GVFs. Finally, there is 1 GVF with uniform random policy with $\gamma = 0$ at a wall event and $\gamma = 0.5$ otherwise.

As a baseline, we also include what we call a \textbf{Forecast} network, which uses $k$-horizon predictions for the hidden state instead of GVFs. The architecture of the Forecast network is otherwise the same as the GVFN. We use a set of horizons $\mathcal{K} = \{1, 2, \ldots, 8\}$, for each of the non-white observations, resulting in a hidden state size of 40. To train these networks online we keep a buffer of $p+\max(\mathcal{K})$ observations, using the first $p$ observations in the BPTT calculation and the next $k$ observations to determine the targets of the network. We then recover the most recent hidden state to train the evaluation GVFs as we would with the RNN and GVFN architectures. More specifically, at time step $t$, we update state $\svec_{t-k}$ with observations $\mathbf{o}_{t-k+1}, \ldots, \mathbf{o}_{t}$.

\begin{figure}[t]
  \center
  % \begin{subfigure}{0.48\textwidth}
  %   \includegraphics[width=\textwidth]{plot_jair/compass_world/raw/compassworld_fig_2.pdf}
  % \end{subfigure}
  % \begin{subfigure}{0.48\textwidth}
  %   \includegraphics[width=\textwidth]{plot_jair/compass_world/raw/compassworld_fig_2_adam.pdf}
  % \end{subfigure}
  % \begin{minipage}{0.35\textwidth}
  \includegraphics[width=0.95\textwidth]{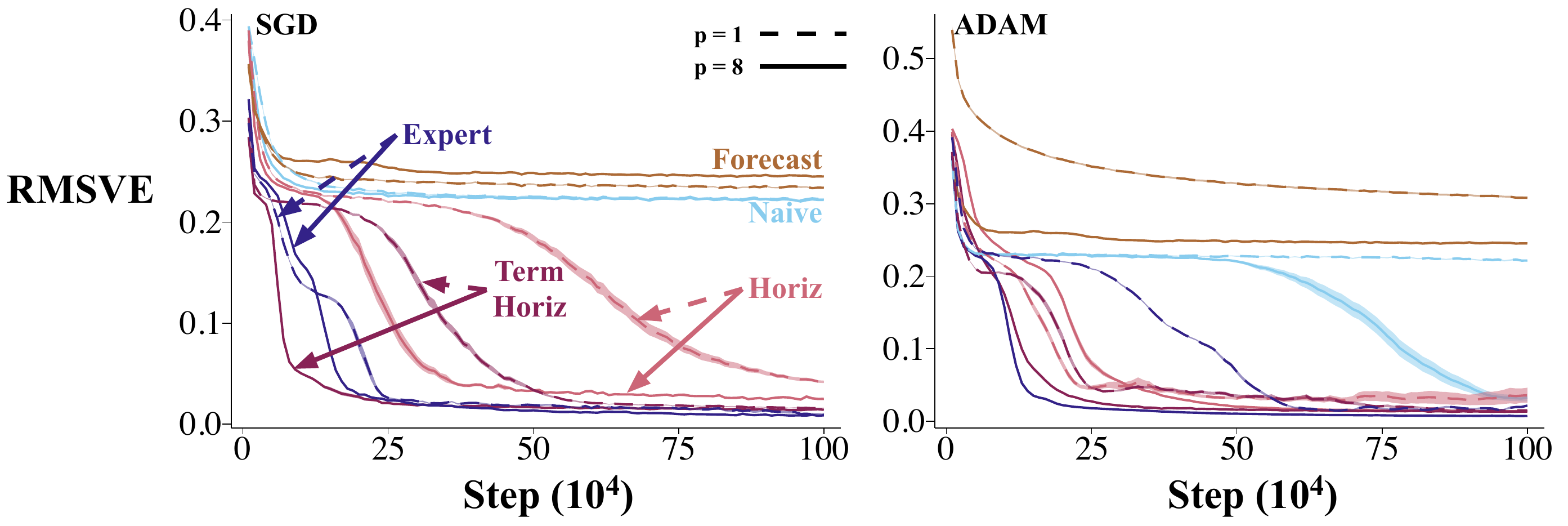}
  \caption{
  Learning curves for \textbf{(dashed)} $p=1$ and \textbf{(solid)} $p=8$ for various GVFN specifications and the Forecast networks. The GVFN is labeled TermHorizon, to highlight that it is composed of terminating horizon GVFs. Learning rates were chosen as in Figure \ref{fig:compass}, where the left plot corresponds to using a constant stepsize and the right to using the ADAM optimizer. The errors were averaged over 30 independent runs, to get the final learning curves $\pm$ standard error.}\label{fig:compass_poor}
\end{figure}

Learning curves for all the GVFN specifications, as well as the Forecast network, with $p = 1$ and $p = 8$, are reported in Figure \ref{fig:compass_poor}. The results indicate that the specification can have a big impact. The two misspecified GVFNs perform noticeably worse than the Terminating Horizon GVFN. As expected, the Naive GVFN is eventually able to learn, with enough steps, $p = 8$ and the ADAM optimizer. It is sufficient to obtain a good state, but poor learnability prevents it from playing a useful role. The Horizon GVFN, which has potentially high magnitude GVF predictions, is closer in performance to the Terminating Horizon GVFN, but clearly worse. The Expert GVFN, on the other hand, can get to a lower error, though it does not have a clear advantage in terms of learning speed or robustness to $p$; this slower learning could again be potentially due to the fact that these expert GVFs were more difficult to learn than the simpler terminating horizon GVFs. Finally, the Forecast network performed very poorly. This is not too surprising in this environment. When considering a $k$-horizon prediction, the target is often zero, with the occasional one. This is generally a hard learning problem, as the resulting prediction loss does not provide a useful constraint. These results clearly show specifying the GVFs used to constrain the hidden state is an important consideration when using GVFNs, and could be the difference between learnable and not learnable representations.

\section{Comparing Recurrent GTD and Recurrent TD}

TD networks with a simple TD network update rule---no backprop through time---have been shown to have divergence issues on a simple six-state domain, called Ringworld~\citep{tanner2005temporal}. In fact, Gradient TD networks \citep{silver2012gradient} were introduced precisely to solve this problem. Because GVFNs are a strict generalization of TD networks, we can set the GVFN to get the same problematic setting if we use a simple TD update (RTD with $p=1$). This raises a natural question of if Recurrent TD (RTD) similarly has divergence issues, and if we need to use Recurrent GTD (RGTD).

In all of our experiments so far, we have opted for the simpler RTD algorithm, rather than the full gradient algorithm RGTD, because empirically we found little difference between the two. RTD, unlike the simple TD update rule, does in fact compute gradients back-in-time, and so should be a more sound update. Further, once we use truncated BPTT, even RGTD is providing a biased estimate of the gradient. But nonetheless RTD---which is built on the semi-gradient TD update---does drop more of the gradient than RGTD. It is likely that RGTD is needed in some cases. But it is possible that for most settings, RTD provides a reasonable interim choice between the simple TD network learning rule, and the more complex RGTD.

In this section, we test RTD and RGTD on Ringworld, to see if they perform differently on this known problematic setting.  Note that for $p=1$, RTD reduces to the simple TD network learning rule, and so we expect poor performance.

Ring World is a six-state domain~\citep{tanner2005temporal} where the agent can move left or right in the ring. All the states are indistinguishable except state six. The observation vector is simply a two bit binary encoding indicating if the agent is in state six or not. The agent behaves uniformly randomly. The goal is to predict the observation bit on the next time step. The environment itself is not too difficult for state-construction; rather a particular TD network causes divergence from the simple TD update rule. The corresponding GVFN consists of two chains of compositional GVFs: one chain for always go right and one chain for always going left. In the first chain, the first GVF is a myopic GVF, that has as cumulant the observed bit after taking action Right, with $\gamma =0$. This first GVF predicts the observation one step into the future. The second GVF has the first GVFs prediction as a cumulant after taking action Right, with $\gamma = 0$. This second GVF predicts the observation two steps into the future. There are five GVFs in each chain, for a total of 10 GVFs in the GVFN.

\begin{figure}[t]
  \center
  \begin{subfigure}{0.55\textwidth}
    \includegraphics[width=\textwidth]{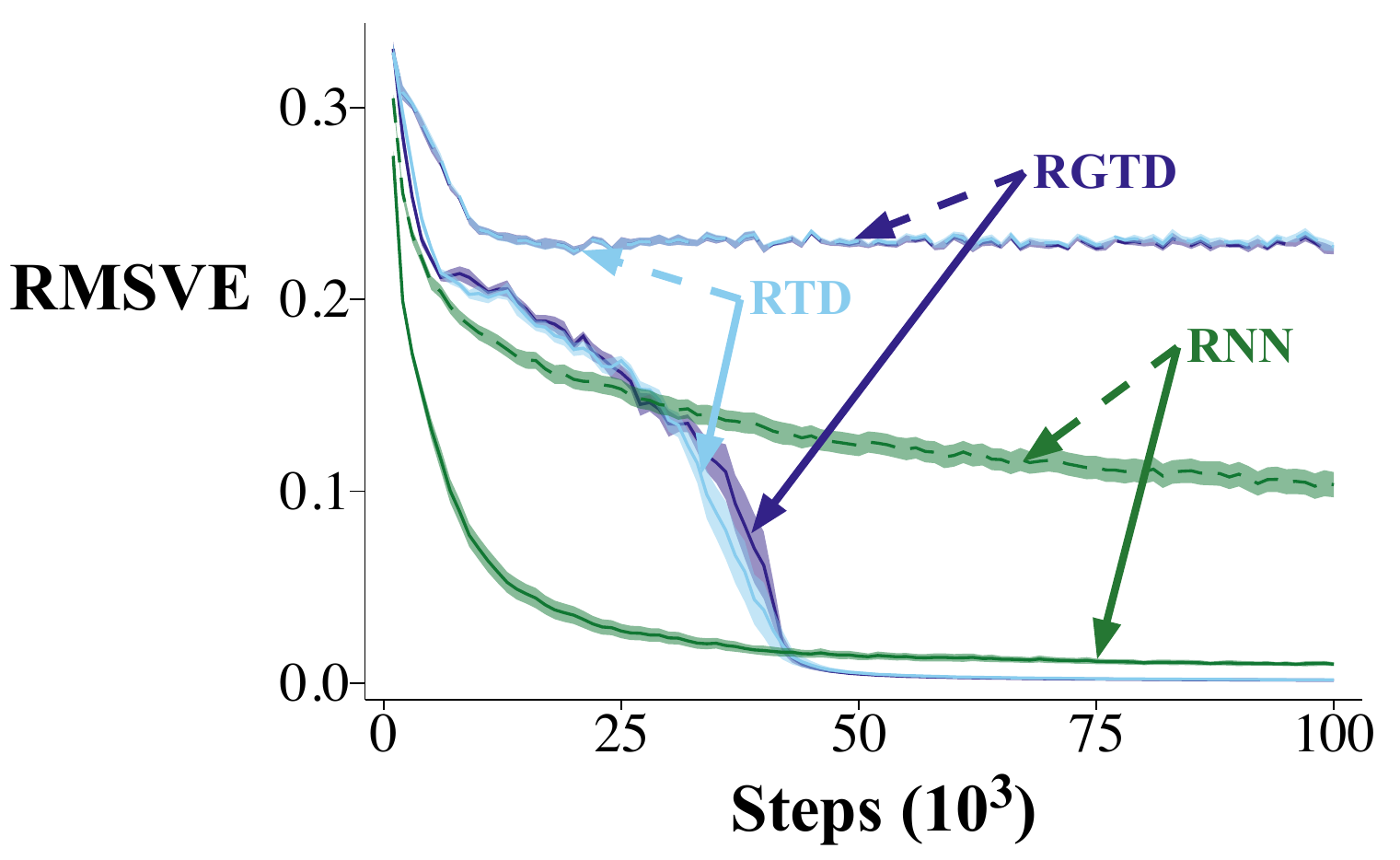}
  \end{subfigure}
  
  \caption{
    Learning curves for $p=1$ and $p=2$ averaged over 10 runs with fixed window smoothing of 1000 steps, in the Ringworld environment. Learning rates chosen from a sweep over $\alpha \in \{0.1\times1.5^i; i\in\{-10, -9, \ldots, 6\}\}$ for the RNN and learning rates $\alpha \in \{0.1\times1.5^i; i\in\{-6, -9, \ldots, 8\}\}$ and $\beta = \{0.0, 0.01\}$ corresponding to RTD and RGTD respectively. All approaches needed only $p = 2$ to learn, including the baseline RNN included for comparison.
  }\label{fig:ring}
\end{figure}

Figure \ref{fig:ring} shows the results of the Ring World experiments for truncation $p=1$ and $p=2$. The GVFNs for both RTD and RGTD needed only $p \ge 2$ to learn effectively. We also include a baseline RNN of the same architecture, that indicates that the GVFN specification does negatively impact performance. But, with even just $ p = 2$, any convergence issues seem to disappear. In fact, RTD and RGTD perform very similarly. The fact that Ringworld is not problematic for RTD is by no means a proof that RTD is sufficient, especially since Ringworld was designed to be a counterexample for the simple TD network update not for RTD. But, it is one more datapoint that RTD and RGTD perform similarly. In future work, we will be investigating a counterexample for RTD, to better understand when it might be necessary to use RGTD.

%%% Local Variables:
%%% mode: latex
%%% TeX-master: "paper"
%%% End:

%% file: appendix.tex
%!TEX root = paper.tex

\section{Algorithmic Details and Derivations} \label{appendix:alg_details_derivs}

In this section we provide the derivation of Recurrent-GTD from the MSPBNE, including off-policy corrections, and the details in recursively calculating the gradients of a GVFN back through time using an RTRL derivation, with details easily extended to BPTT. We also briefly discuss the forecast networks used briefly in the main text.

\subsection{Re-expressing the MSPBNE} \label{appendix:mspbne}

The MSPBNE was derived for on-policy prediction questions, for TD Networks \citep{silver2012gradient}. The main extension is to allow for (1) off-policy prediction, which is straightforward to do using importance sampling ratios and (2) extension to continuation functions. Note that we assume that the target policies $\pi_i$ do not change as the state estimate changes; rather, they are functions of history. We first show the result without importance sampling ratios, as they are implicit in the expectations. We then provide a corollary with explicit importance sampling ratios.

\mspbnelemma*
 \begin{proof}
Starting with equation \eqref{eq_projform} and for $\Delta_{\weights} \defeq \Bn \vinone_\weights - \vinone_\weights$, we get
   \begin{align*}
     \text{MSPBNE}(\weights)
     & = \| \Pi_{\weights} \Bn \vinone_\weights - \vinone_\weights \|_{\dw}^2\\
     & = \| \Pi_{\weights} \left[\Bn\vinone_\weights - \vinone_\weights \right] \|_{\dw}^2\\
     & = \| \Pi_{\weights}\Delta_{\weights} \|_{\dw}^2
   \end{align*}
   We can wrap the projection operator around the full TD error $\Delta_{\weights}$, because it has no affect on $\vinone_\weights$ which is already in the space. We then plug in the definition of $\Pi_\weights$
   \begin{align}
     \Pi_\weights^\top \dwdiag \Pi_\weights
         &= \dwdiag^\top \phimat_\weights (\phimat_\weights^\top \dwdiag \phimat_\weights)^\inv \phimat^\top \dwdiag \nonumber \\
     \| \Pi_{\weights}\Delta_{\weights} \|_{\dw}^2
         &= \Delta_\weights^\top \Pi_\weights^\top \dwdiag \Pi_\weights \Delta_\weights \nonumber \\
    &= \Delta_\weights^\top \dwdiag^\top \phimat_\weights (\phimat_\weights^\top \dwdiag \phimat_\weights)^\inv \phimat^\top \dwdiag \Delta_\weights \label{mspbne_mat}
   \end{align}
   As in prior gradient TD work we then convert the matrix operations to expectation forms.
   \begin{align*}
     \phimat_\weights^\top \dwdiag \phimat_\weights &= \sum_{j=1}^n \sum_{\hvec\in\Hists} \dw(\hvec) \phivec_{j,\weights}(\hvec) \phi_{j, \weights}(\hvec)^\top = \Expected_d\left[\sum_{j=1}^n \phivec_{j,\weights}(H)\phivec_{j,\weights}(H)^\top\right]\\
     &= W(\weights)\\
     \phimat_\weights^\top \dwdiag \Delta_\weights &= \sum_{j=1}^n \sum_{\hvec\in\Hists} \dw(\hvec) \phivec_{j,\theta}(\hvec) \sum_{a\in\Actions} \pi_j(a|\hvec) \Expected[\delta_j(\hvec,a,H')] = \sum_{j=1}^n\Expected_{d,\pi_j}\left[\delta_j(H,A,H')\phivec_{j,\weights}(H)\right]\\
     &=  \boldsymbol{\delta}(\weights)
   \end{align*}
   Then substituting into equation \eqref{mspbne_mat}, we get the result $\text{MSPBNE}(\weights) = \boldsymbol{\delta}(\weights)^\top W(\weights)^\inv \boldsymbol{\delta}(\weights)$.
 \end{proof}

Now we do not actually get samples according to $\pi_j$; instead, we get them according to the behaviour $\mu$. Throughout this work, we have assumed a coverage property for $\mu$. This means that the behaviour policy $\mu$ satisfies $\mu(a | \hvec) > 0$ if any $\pi_j(a | \hvec) > 0$ for policies $\pi_1, \ldots, \pi_\numgvfs$. 

\begin{corollary}\label{col:mspbne_is}
For importance sampling ratios $\rho_j(a | \hvec) \defeq \frac{\pi_j(a | \hvec)}{\mu(a | \hvec)}$ and
  \begin{align*}
    \boldsymbol{\delta}_\mu(\weights) &\defeq \Expected_{d, \mu}\bigg[\sum_{j=1}^\numgvfs \rho_j(H,A) \tderror_j(H,A,H') \phivec_{j,\weights}(H) \bigg]\\
         &= \sum_{\hvec \in \Hists} d(\hvec) \sum_{a \in \Actions} \mu(a|\hvec) \sum_{j=1}^\numgvfs  \rho_j (a | \hvec) \Expected\bigg[\tderror_j(H,A,H') \phivec_{j,\weights}(\hvec) | H = \hvec, A = a \bigg] 
  \end{align*} 
  then we can show that $\boldsymbol{\delta}_\mu(\weights) = \boldsymbol{\delta}(\weights) $ and so we can write
   \begin{align*}
    \text{MSPBNE}(\weights) &= \boldsymbol{\delta}_\mu(\weights)^\top W(\weights)^\inv  \boldsymbol{\delta}_\mu(\weights)
   \end{align*}
\end{corollary}
\begin{proof}
The key is simply to show that $\boldsymbol{\delta}_\mu(\weights) = \boldsymbol{\delta}(\weights)$, because $W(\weights)$ depends only on $d$, not on the policies $\pi$ or $\mu$. This is straightforward with the typical cancellation in importance sampling ratios
  \begin{align*}
   \boldsymbol{\delta}_\mu(\weights)
   &= \sum_{\hvec \in \Hists} d(\hvec) \sum_{a \in \Actions} \mu(a|\hvec) \sum_{j=1}^\numgvfs  \rho_j (a | \hvec) \Expected\bigg[\tderror_j(\hvec, a, H') \phivec_{j,\weights}(h) | H = h, A = a \bigg] \\
     &= \sum_{\hvec \in \Hists} d(\hvec)  \sum_{j=1}^\numgvfs  \sum_{a \in \Actions} \mu(a|\hvec) \rho_j (a | \hvec) \Expected\bigg[\tderror_j(\hvec, a, H') \phivec_{j,\weights}(h) | H = h, A = a \bigg] \\ 
     &= \sum_{\hvec \in \Hists} d(\hvec)  \sum_{j=1}^\numgvfs  \sum_{a \in \Actions} \pi_j (a | \hvec) \Expected\bigg[\tderror_j(\hvec, a, H') \phivec_{j,\weights}(h) | H = h, A = a \bigg] \\ 
&= \boldsymbol{\delta}(\weights)         .
  \end{align*}
\end{proof}
From here on, therefore, we assume that $\boldsymbol{\delta}(\weights)$ is defined more generally as the above $\boldsymbol{\delta}_\mu(\weights)$, since they result in the same objective but this more general expression more obviously highlights off-policy sampling. 

\newcommand{\Hess}{\nabla_\weights^2}

\subsection{Deriving Recurrent-GTD} \label{appendix:alg_derivs}

 Now that the objective is written in its expectation form, the gradients can be take with respect to the weight parameter. The main body stated the result for a simplified setting (Theorem \ref{thm:gradients}), to make it simpler to understand the result. We provide the more general result here, for compositional GVFs.

\begin{restatable}{theorem}{gradtheoremgen}\label{thm:gradientsgen}
% \begin{theorem}
  Assume that $V_{\weights}(\hvec)$ is twice continuously differentiable as a function of $\weights$ for all
  histories $\hvec\in\Hist$ where $\dw(\hvec)>0$ and that $W(\cdot)$, defined in Equation \eqref{eqn_w}, is non-singular in a small neighbourhood of $\weights$. Then for
  \begin{align*}
    \boldsymbol{\delta}(\weights)
      &\defeq 
        \Expected_{d,\mu}\bigg[ \sum_{j=1}^\numgvfs \rho_j(H,A) \tderror_j(H,A,H') \phivec_{j,\weights}(H) \bigg] \\
    \wvec(\weights)
      &= W(\weights)^\inv\boldsymbol{\delta}(\weights) \\
     \psivec(\weights) &= \Expected_{d, \mu}\left[\sum\limits_{j=1}^{\numgvfs} \Big(\rho_j(H,A)\delta_j(H,A,H') - \phivec_{j,\weights}(H)^\trans  \wvec(\weights)\Big)  \nabla^2 V^{(j)}_{\weights}(H)  \wvec(\weights) \right]  
  \end{align*}
we get the gradient
 \begin{align}
  & -\frac{1}{2} \nabla  \textrm{MSPBNE}(\weights)
    =
     -\Expected_{d, \mu}\bigg[ 
     \sum\limits_{j=1}^\numgvfs  \rho_j(H,A) \nabla_\weights \delta_j(H,A,H') \phivec_{j,\theta}(H)^\top  \bigg] \wvec(\weights)  - \psivec(\weights) \label{eq_gtd2}\\
    &=
    \boldsymbol{\delta}(\weights) - \psivec(\weights) \label{eq_tdc}\\
      & \ \ \ \ - \Expected_{d,\mu}\bigg[ \sum\limits_{j+1}^\numgvfs  \rho_j(H,A) \bigg[\sum_{i=1}^\numgvfs \cfunc(j,i) \phivec_{i,\weights}(H) + \gamma_j(H,A,H') \phivec_{j,\weights}(H')\bigg] \phivec_{j,\weights}(H)^\trans \wvec(\weights) \bigg]    \nonumber
 \end{align}
\end{restatable}

\begin{proof}
For simplicity in notation below, we drop the explicit dependence on the random
variable $H$ in the expectations. 
\begin{align*}
  \phivec_{j,\weights}(H) \rightarrow \phivec_{j,\weights}&,\hspace{1cm}
  \phivec_{j,\weights}(H') \rightarrow \phivec_{j,\weights}'\\
  \tderror_j(H,A,H') \rightarrow \tderror_j&,\hspace{1cm}\rho_j(H,A) \rightarrow \rho_j
\end{align*}
Further, we will use $\partial_i$ to indicate the partial derivative with respect to $\weights_i$. 
We also assume all expectations are with respect to $d, \text{ and } \mu$. We use $J$ to denote the MSPBNE, which from Lemma \ref{lemma:mspbne_exp} and Corollary \ref{col:mspbne_is}, can be written $J(\weights) = \boldsymbol{\delta}(\weights)^\top W(\weights)^\inv \boldsymbol{\delta}(\weights)$.
When applying the product rule
\begin{align*}
  \partial_i J(\weights)
  &= 2 (\partial_i \boldsymbol{\delta}(\weights))^\top \wvec(\weights) + \boldsymbol{\delta}(\theta)^\top  \partial_i W(\weights)^\inv \boldsymbol{\delta}(\theta) \\
   \partial_i \boldsymbol{\delta}(\weights)
  &= \Expected\bigg[\sum_{j=1}^\numgvfs \rho_j   \partial_i\phivec_{j,\weights} \tderror_j + \phivec_{j,\weights}  \partial_i\tderror_j \bigg] \\ 
   \partial_i W(\weights)^\inv
  &= - W(\weights)^\inv  \partial_i W(\weights) W(\weights)^\inv 
  = -2 W(\weights)^\inv \Expected\bigg[\sum_{j=1}^\numgvfs (\partial_i\phivec_{j,\weights}) \phivec_{j,\weights}^\trans \bigg] W(\weights)^\inv 
  \end{align*}
Recall that $\wvec(\weights) = W(\weights)^\inv\boldsymbol{\delta}(\weights)$, and that $W(\weights)$ is symmetric, giving
  \begin{align*}
 \boldsymbol{\delta}(\theta)^\top \partial_i W(\weights)^\inv \boldsymbol{\delta}(\theta)
 &= -2\boldsymbol{\delta}(\theta)^\top W(\weights)^\inv \Expected\bigg[\sum_{j=1}^\numgvfs (\partial_i\phivec_{j,\weights})\phivec_{j,\weights}^\trans \bigg] W(\weights)^\inv  \boldsymbol{\delta}(\theta)\\
  &= -2\wvec(\weights)^\trans \Expected\bigg[\sum_{j=1}^\numgvfs (\partial_i\phivec_{j,\weights}) \phivec_{j,\weights}^\trans \bigg] \wvec(\weights)\\
  &= -2\wvec(\weights)^\trans \Expected\bigg[\sum_{j=1}^\numgvfs \phivec_{j,\weights} (\partial_i\phivec_{j,\weights})^\trans \bigg] \wvec(\weights)
  \end{align*}
   The last line follows from the fact that the transpose of a scalar is equal to the scalar. Here we transpose the whole expression, leading to a transpose of the outer-product inside the sum.
  Additionally,
 \begin{align*}
 \partial_i \boldsymbol{\delta}(\weights)^\top \wvec(\weights) 
  &= \Expected\bigg[\sum_{j=1}^\numgvfs \rho_j   \tderror_j (\partial_i\phivec_{j,\weights}) + \rho_j\phivec_{j,\weights}  \partial_i\tderror_j \bigg]^\trans  \wvec(\weights)\\ 
&=   \Expected\bigg[\sum_{j=1}^\numgvfs \rho_j \tderror_j (\partial_i\phivec_{j,\weights})^\trans \bigg] \wvec(\weights)
    + \Expected\bigg[ \sum_{j=1}^\numgvfs \rho_j \partial_i\tderror_j  \phivec_{j,\weights}^\trans  \bigg] \wvec(\weights)
  \end{align*}
  Grouping the terms with $(\partial_i\phivec_{j,\weights})$, we get
 \begin{align*}
&\Expected\bigg[\sum_{j=1}^\numgvfs \rho_j \tderror_j (\partial_i\phivec_{j,\weights})^\trans \bigg] \wvec(\weights) - \wvec(\weights)^\trans \Expected\bigg[\sum_{j=1}^\numgvfs \phivec_{j,\weights}(\partial_i\phivec_{j,\weights})^\trans \bigg] \wvec(\weights)\\
&= \Expected\bigg[\sum_{j=1}^\numgvfs \Big( \rho_j \tderror_j - \wvec(\weights)^\trans\phivec_{j,\weights} \Big)(\partial_i\phivec_{j,\weights})^\trans\wvec(\weights)\bigg] \\
&= \boldsymbol{\psi}_i(\weights)
  \end{align*}
where the last follows from the definition of $\nabla_\weights \boldsymbol{\psi}(\weights)$, which is the gradient vector composed of partial derivatives $\boldsymbol{\psi}_i(\weights)$. 
 Therefore,
\begin{align*}
  \partial_i J(\weights)
  &= 2 \partial_i \boldsymbol{\delta}(\weights)^\top \wvec(\weights) + \boldsymbol{\delta}(\theta)^\top  \partial_i W(\weights)^\inv \boldsymbol{\delta}(\theta) \\
&= 2\boldsymbol{\psi}_i(\weights) + 2\Expected\bigg[ \sum_{j=1}^\numgvfs \rho_j \partial_i\tderror_j \phivec_{j,\weights}^\trans  \wvec(\weights)  \bigg]
  \end{align*}
 which proves Equation \eqref{eq_gtd2}. 
  Now we can further simplify the second term, using the fact that $\phivec_{j,\weights} = \nabla_\weights V_{j,\weights}$, giving
    \begin{align*}
\nabla_\weights \tderror_j = \nabla_\weights c_{j,\weights} + \gamma_j \phivec_{j,\weights}' - \phivec_{j,\weights}
.
  \end{align*}
 Now notice that 
\begin{align*}
\Expected\bigg[ \sum_{j=1}^\numgvfs \rho_j \nabla_\weights\tderror_j \phivec_{j,\weights}^\trans  \wvec(\weights)  \bigg] 
&= \Expected\bigg[ \sum_{j=1}^\numgvfs \rho_j \big(\nabla_\weights c_{j,\weights} + \gamma_j \phivec_{j,\weights}' - \phivec_{j,\weights}\big) \phivec_{j,\weights}^\trans  \wvec(\weights)  \bigg] \\
&= -\Expected\bigg[ \sum_{j=1}^\numgvfs \rho_j \phivec_{j,\weights}\phivec_{j,\weights}^\trans \bigg]\wvec(\weights) + \Expected\bigg[ \sum_{j=1}^\numgvfs \rho_j \big(\nabla_\weights c_{j,\weights} + \gamma_j \phivec_{j,\weights}'\big) \phivec_{j,\weights}^\trans  \wvec(\weights)  \bigg] 
  \end{align*}
 Because $\wvec(\weights) = W(\weights)^\inv\boldsymbol{\delta}(\weights)$, 
 \begin{align*}
\Expected\bigg[ \sum_{j=1}^\numgvfs \rho_j \phivec_{j,\weights}\phivec_{j,\weights}^\trans \bigg]\wvec(\weights) 
=
W(\weights)\wvec(\weights) = \boldsymbol{\delta}(\weights)
  \end{align*}
 Putting this all together, we get that
    \begin{align*}
  -\tfrac{1}{2} \nabla_\weights J(\weights)
  &= -\boldsymbol{\psi}_i(\weights) - \Expected\bigg[ \sum_{j=1}^\numgvfs \rho_j  \nabla_\weights\tderror_j \phivec_{j,\weights}^\trans\wvec(\weights) \bigg]\\
  &= -\boldsymbol{\psi}(\weights) + \boldsymbol{\delta}(\weights) -\Expected\bigg[ \sum_{j=1}^\numgvfs \rho_j \big(\nabla_\weights c_{j,\weights} + \gamma_j \phivec_{j,\weights}'\big) \phivec_{j,\weights}^\trans  \wvec(\weights)  \bigg] 
\end{align*}
completing the proof.
\end{proof}

The resulting Recurrent GTD algorithm explicitly learns a second set of weights $\wvec$, to perform this update. In our implementation, we use a particular form of composition, namely that the cumulant for a GVF is a linear weighting of the predictions of some of the other GVFs on the next time step. If we let $c(i, j)$ indicate the weight on the $i$th GVF in the cumulant for the $j$th GVF, then we get that $\nabla_\weights c_{j,t} =   \sum_{i=1}^\numgvfs c(j,i) \phivec_{i,t}'$.

The {\bf Recurrent GTD} update is 
\begin{align}
\svec_t &\gets f_{\weights_t}(\svec_{t-1}, \xvec_t) \nonumber\\
\svec_{t+1} &\gets f_{\weights_t}(\svec_{t}, \xvec_{t+1}) \nonumber\\
\phivec_{t,j} &\gets \nabla_\weights \svec_{t,j} \hspace{2.0cm} \triangleright \text{ Compute sensitivities using truncated BPTT}  \nonumber\\
\phivec'_{t,j} &\gets \nabla_\weights \svec_{t+1,j}  \nonumber\\
\rho_{t,j} &\gets \frac{\pi_j(a_t | \obs_t)}{\mu(a_t | \obs_t)}  \nonumber\\
\vvec_t &= \nabla^2 \svec_t \wvec_t \hspace{2.0cm} \triangleright \text{ Computed using R-operators, see Appendix \ref{app_gradients}} \nonumber\\
  \psivec_t &= \sum_{j=1}^{\numgvfs} ( \rho_{j,t}\delta_{j,t} - \phivec_{j,t}^\trans  \wvec_t)  \vvec_t \label{eq_rgtd_gen}\\
  \weights_{t+1} &= \weights_{t} + \alpha_t \bigg[ \sum_{j=1}^{\numgvfs}  \rho_{j,t} \tderror_{j,t} \phivec_{j,t} - \rho_{j,t} \bigg[\nabla_\weights c_{j,t} + \gamma_{j,t+1} \phivec'_{j,t}  \bigg] \phivec_{j,t}^\trans \wvec_t - \psivec_t  \bigg] \nonumber\\
   \wvec_{t+1} &= \wvec_t + \beta_t \bigg[ \sum_{j=1}^{\numgvfs}  \rho_{j,t} \Big(\tderror_{j,t} - \phivec_{j,t}^\trans \wvec_t\Big) \phivec_{j,t} \bigg] \nonumber
\end{align}

\subsection{Computing gradients of the value function back through time}\label{app_gradients}

\newcommand{\Feats}{\zvec}

In this section, we show how to compute $\phivec_t$, which was needed in the
algorithms. Recall from Section \ref{sec_algs} that
we set $V^{(j)}(\svec_{t+1}) = \svec_{t+1,j}$, and using
$\Feats_{t+1}\defeq\twovec{\svec_{t}}{\xvec_{t+1}}$ let
  $\svec_{t+1,j} =
  \sigma\left( \Feats_{t+1}^\top\weights^{(j)}\right)$ for
  some activation function $\sigma:\RR\rightarrow\RR$. 
  For both Backpropagation Through Time or Real Time Recurrent Learning,
  it is useful to take advantage of the following formula for \emph{recurrent sensitivities}
\begin{align*}
  \pd{V^{(i)}(S_{t+1})}{\weights_{(k,j)}} &= \actdot(\Feats_{t+1}^\trans \weights^{(i)}) \biggr(\biggr(\pd{\Feats_{t+1}}{\weights_{(k,j)}} \biggr)^\trans \weights^{(i)} + (\Feats_{t+1})_j \krondelta_{i,k}\biggr) \\
  &= \actdot(\Feats_{t+1}^\trans \weights^{(i)}) \left(\biggr[\pd{V^{(1)}(S_{t})}{\weights_{(k,j)}}, ... ,\pd{V^{(n)}(S_{t})}{\weights_{(k,j)}},\zerovec^\top\biggr] \weights^{(i)} + (\Feats_{t+1})_j \krondelta_{i,k}\right) 
\end{align*}
where $\krondelta$ is the Kronecker delta function
and $\actdot(\cdot)$ is shorthand for
the derivative of $\sigma$ w.r.t its scalar input.
 Given this formula, BPTT or
RTRL can simply be applied.

For Recurrent GTD---though not for Recurrent TD---we additionally need to compute the Hessian back in time, for the Hessian-vector product. The Hessian for each value function is a $\numgvfs(\featuresize)\times \numgvfs(\featuresize)$ matrix; computing the Hessian-vector product naively would cost at least $O((\featuresize + \numgvfs)^2 \numgvfs^2)$ for each GVF, which is prohibitively expensive. We can avoid this using R-operators also known as Pearlmutter's method \citep{pearlmutter1994fast}. 

The R-operator $\Roperator\{\cdot\}$ is defined as 
\begin{equation*}
  \mathcal{R}_\wvec\biggr\{\gvec(\weights)\biggr\} \defeq \frac{\partial \gvec(\weights + r \wvec)}{ \partial r} \biggr\rvert_{r=0}
\end{equation*}
for a (vector-valued) function $\gvec$ and satisfies 
\begin{equation*}
  \mathcal{R}_\wvec\biggr\{\nabla_\weights f(\weights)\biggr\} = \nabla^2_\weights f(\weights) \wvec.
\end{equation*}
Therefore, instead of computing the Hessian and then producting with $\wvec_t$, this operation can be completed in linear time, in the length of $\wvec_t$. 

Specifically, for our setting, we have
\begin{align*}
  \mathcal{R}_w\biggr\{\actdot(\Feats_t^\top\weights)[\nabla_\weights \Feats_t^\trans \weights + \Feats_t^\trans \nabla_\weights\weights]\biggr\}
  &= \pd{}{r}\biggr(\actdot(\Feats_t^\top(\weights + r \wvec)[\nabla_\weights \Feats_t^\trans (\weights + r\wvec) + \Feats_t^\trans \nabla_\weights(\weights+r\wvec)]\biggr) \biggr\rvert_{r=0}
\end{align*}
To make the calculation more managable we seperate into each partial for every node k and associated weight j.
\begin{align*}
  \pd{V^{(i)}(S_{t+1}, \weights)}{\weights_{(k,j)}} &= \actdot(\Feats_{t+1}^\trans \weights^{(i)}) (\eta_{t+1})_{i,k,j} \\
  (\eta_{t+1})_{i, k, j} &= ((\valuedtheta_t)_{k,j}^\trans \weights^{(i)} + (\Feats_{t+1})_j \delta_{i,k}) \\
  (\valuedtheta_t)_{k,j} &= \left[\pd{V^{(1)}(S_{t})}{\weights_{(k,j)}}, ... ,\pd{V^{(n)}(S_{t})}{\weights_{(k,j)}},\zerovec^\top\right]^\top\\
  \valuedr_t &= \left[\pd{V^{(1)}(S_{t})}{r}, ... ,\pd{V^{(n)}(S_{t})}{r},\zerovec^\top\right]^\top \\
  \end{align*}
  \begin{align*}
  \RopValueVect &= \left[\mathcal{R}_w\biggr\{\pd{V^{(1)}(S_{t-1})}{\weights_{(k,j)}}\biggr\}, ..., \mathcal{R}_w\biggr\{\pd{V^{(\numgvfs)}(S_{t-1})}{\weights_{(k,j)}}\biggr\},\zerovec^\top\right]^\top \\
  \mathcal{R}_w\left\{\pd{V^{(i)}(S_{t+1}, \weights)}{\weights_{(k,j)}}\right\} &= \frac{\partial^2 V^{(i)}(S_{t+1}, \weights + r\secweights)}{\partial r \partial \weights_{(k,j)}} \biggr\rvert_{r=0} \\
  &= \actdotdot\biggr(\Feats_{t+1}^\trans (\weights^{(i)} + r\secweights_i)\biggr) \biggr(\valuedr_t^\trans (\weights^{(i)} + r\secweights_i) + \Feats_{t+1}^\trans \secweights_i\biggr) (\eta_{t+1})_{i,k,j} \\
  &\phantom{{}=}
   + \actdot\biggr(\Feats_{t+1}^\trans (\weights^{(i)} + r\secweights_i)\biggr)\biggr(\RopValueVect^\trans (\weights^{(i)} + r\secweights)
   % &\phantom{{}= + \actdot(x_{t+1}^\trans (\weights^{(i)} + r\secweights_i))}
  + (\valuedtheta_t)_{k,j}^\trans w_i
  + (\valuedr_t)_j \krondelta_{k,i}\biggr) \biggr\rvert_{r=0} \\
  &= \actdotdot\biggr(\Feats_{t+1}^\trans \weights^{(i)}\biggr) \biggr(\valuedr_t^\trans (\weights^{(i)}) + \Feats_{t+1}^\trans \secweights_i\biggr) (\eta_{t+1})_{i,k,j} \\
  &\phantom{{}=}
   + \actdot\biggr(\Feats_{t+1}^\trans \weights^{(i)}\biggr) \biggr(\RopValueVect^\trans \weights^{(i)}
   % &\phantom{{}=}
  + (\valuedtheta_t)_{k,j}^\trans \secweights_i + (\valuedr_t)_j \krondelta_{k,i}\biggr)\\
  \pd{V^{(i)}(S_{t})}{r} &= \actdot(\Feats_t^\trans \weights^{(i)})(\valuedr_{t-1}^\trans \weights^{(i)} + \Feats_t^\trans w_i)
  % \krondelta_{k,i} &\defeq \text{Kronecker Delta} 
\end{align*}

\subsection{TD($\lambda$) for learning GVFNs}\label{app_tdlambda}

For many of the experiments we used Recurrent TD with no back-propagation through time $p=1$. This algorithm only adjusts parameters to minimize immediate TD error. In many cases, this was sufficient, but at times it was slow and increasing $p$ improved learning. Another strategy is to use traces to obtain credit assignment back-in-time. The TD-error on this step can be attributed to state values back-in-time, with the \textbf{TD($\boldsymbol{\lambda}$) algorithm} 
\begin{align}
\svec_t &\gets f_{\weights_t}(\svec_{t-1}, \xvec_t) \nonumber\\
\svec_{t+1} &\gets f_{\weights_t}(\svec_{t}, \xvec_{t+1}) \nonumber\\
\gvec_{t,j} &\gets \nabla_{\weights_j} f_{\weights_t}(\svec_{t-1}, \xvec_{t}) && \triangleright \text{ gradient given $\svec_{t-1}$, no BPTT} \nonumber\\
\evec_{t,j} &\gets \gvec_{t,j} + \gamma_{t,j} \lambda \evec_{t-1,j} && \triangleright \text{ eligibility trace, $0 \le \lambda \le 1$} \nonumber\\
\delta_{t,j} &\gets C_{t+1}^{(j)} + \gamma_{t+1, j} \svec_{t+1,j} - \svec_{t,j}   \nonumber\\
\weights_{t+1,j} &\gets \weights_{t,j} + \alpha_t \tderror_{t,j} \evec_{t,j} \label{eq_td_lambda}
\end{align}
Notice the difference to Recurrent TD and Recurrent GTD, that the weights for each GVF are updated independently. This difference arises because the gradient computations for back-in-time, for the sensitivities, is what couples the updates. Without these sensitivities, the immediate gradient of the value $\gvec_{t,j}$ is independent for each GVF. 

%%% Local Variables:
%%% mode: latex
%%% TeX-master: "paper"
%%% End: